\documentclass{article}

\PassOptionsToPackage{numbers, compress}{natbib}
\usepackage[final]{neurips_2023}

\usepackage[utf8]{inputenc} %
\usepackage[T1]{fontenc}    %
\usepackage[colorlinks=true,linkcolor=black,citecolor=black,urlcolor=black]{hyperref}       %
\usepackage{url}            %
\usepackage{booktabs}       %
\usepackage{amsfonts}       %
\usepackage{nicefrac}       %
\usepackage{microtype}      %
\usepackage{wrapfig}
\usepackage{bbm}

\usepackage{graphicx}
\usepackage[dvipsnames]{xcolor}
\usepackage{mathtools}
\usepackage{amsmath}
\usepackage{amssymb}
\usepackage{amsthm}
\usepackage{xfrac}
\usepackage{bm}
\usepackage[inline]{enumitem}
\usepackage{algorithm}
\usepackage{algpseudocode}
\usepackage{ifthen}
\usepackage[font=small]{subfig}
\usepackage{float}
\usepackage{wrapfig}
\usepackage{csquotes}
\usepackage{threeparttable}
\usepackage[capitalize,noabbrev]{cleveref}
\usepackage{thm-restate}
\usepackage{colortbl}
\usepackage[abbreviations]{foreign}
\usepackage[binary-units]{siunitx}
\usepackage{multirow}
\usepackage{makecell}

\newcolumntype{b}{>{\columncolor{green!20}}c}
\graphicspath{ {figures/} }
\newboolean{showcomments}
\setboolean{showcomments}{false}
\ifthenelse{\boolean{showcomments}}
{ \newcommand{\mynote}[3]{
		\fbox{\bfseries\sffamily\scriptsize#1}
		{\small$\blacktriangleright$\textsf{\emph{\color{#3}{#2}}}$\blacktriangleleft$}}
	\newcommand{\zzz}[1]{{\setlength{\fboxsep}{2pt}\fcolorbox{black}{yellow}{\textsf{\emph{#1}}}}\xspace}}
{ \newcommand{\mynote}[3]{}
	\newcommand{\zzz}[1]{}}

\newcommand{\rp}[1]{\mynote{Rafael}{#1}{orange}}

\newcommand{\TODO}[1]{\zzz{TODO: #1}}

\newcommand{\newtext}[1]{#1}
\newcommand{\sys}{\newalgorithm{}\xspace}

\newcommand{\femnist}{FEMNIST\xspace}

\newcommand{\cifar}{CIFAR-10\xspace}

\newcommand{\sgd}{{\xspace}\ac{SGD}\xspace}
\newcommand{\dpsgd}{{\xspace}\ac{D-PSGD}\xspace}
\newcommand{\iid}{\ac{IID}\xspace}

\newcommand{\elo}{EL-Oracle\xspace}
\newcommand{\ello}{EL-Local\xspace}

\newtheorem{assumption}{Assumption}
\newtheorem{remark}{Remark}

\newcommand{\samplenum}{s}

\newcommand{\expect}[1]{\mathop{{}\mathbb{E}}\left[{#1}\right]}
\newcommand{\condexpect}[2]{\mathbb{E}_{#1}\left[{#2}\right]}

\providecommand{\iprod}[2]{\ensuremath{\left\langle #1,\,#2  \right\rangle}}

\newcommand{\card}[1]{\left\lvert{#1}\right\rvert}

\newcommand{\norm}[1]{\left\lVert{#1}\right\rVert}

\newcommand{\indexvar}[3]{{#3}^{\ifthenelse{\equal{#1}{}}{}{\left({#1}\right)}}_{#2}}

\newcommand{\loss}{ F}
\newcommand{\modelp}[2]{\indexvar{#1}{#2}{y}}

\newcommand{\lossperpoint}{f}
\newcommand{\heterparam}{\mathcal{H}}

\newcommand{\receivedsubset}[2]{\indexvar{#1}{#2}{\mathcal{S}}}

\newcommand{\gradient}[2]{\indexvar{#1}{#2}{g}}
\newcommand{\model}[2]{\indexvar{#1}{#2}{x}}

\newcommand{\avgmodel}[1]{\bar{x}_{#1}}
\newcommand{\localstep}{\gamma}

\newcommand{\newalgorithm}{EL}
\newcommand{\contractionp}[1]{\alpha_{#1}}
\newcommand{\contraction}[1]{\beta_{#1}}
\newcommand{\contractionb}[1]{\eta_{#1}}

\newcommand{\localloss}[1]{\indexvar{#1}{}{f}}
\newcommand{\dist}[1]{\indexvar{#1}{}{\mathcal{D}}}
\newcommand{\datapoint}[2]{\indexvar{#1}{#2}{\xi}}

\renewcommand{\paragraph}[1]{\textbf{#1}~}

\newtheorem{theorem}{Theorem}
\newtheorem{lemma}{Lemma}

\renewcommand{\paragraph}[1]{\textbf{#1}~}

\def\R{\mathbb{R}}

\newcommand{\avggrad}[1]{\indexvar{}{#1}{\overline{g}}}
\newcommand{\AvgGrad}[1]{\overline{\nabla F}_{#1}}
\newcommand{\localgrad}[1]{\nabla{} \indexvar{#1}{}{\localloss{}}}

\newcommand{\dataspace}{\mathcal{Z}}

\usepackage{acronym}
\acrodef{DL}{decentralized learning}
\acrodef{ML}{machine learning}
\acrodef{D-PSGD}{decentralized parallel stochastic gradient descent}
\acrodef{FL}{federated learning}
\acrodef{FI}{federated inference}
\acrodef{SGD}{stochastic gradient descent}
\acrodef{IID}{independent and identically distributed}
\acrodef{non-IID}{non independent and identically distributed}
\acrodef{RMSE}{root mean square error}
\acrodef{RMW}{random model walk}
\acrodef{GL}{gossip learning}
\acrodef{DWT}{discrete wavelet transform}
\acrodef{LAN}{local area network}
\acrodef{WAN}{wide area network}
\acrodef{NN}{neural network}
\acrodef{KD}{knowledge distillation}
\acrodef{EL}{Epidemic Learning}
\acrodef{EL-Oracle}{\Ac{EL}-Oracle}
\acrodef{EL-Local}{\Ac{EL}-Local}
\usepackage{tikz}
\usepackage{pgfplots}
\usepackage[eulergreek]{sansmath}

\pgfplotsset{compat=newest}
\usepgfplotslibrary{external,units,colorbrewer,groupplots,fillbetween}
\tikzsetexternalprefix{figures/}
\tikzset{external/mode=list and make}
\usetikzlibrary{patterns}

\makeatletter
\begingroup\endlinechar=-1\relax
\everyeof{\noexpand}%
\edef\x{\endgroup\def\noexpand\homepath{%
        \@@input|"kpsewhich --var-value=HOME" }}\x
\makeatother

\def\overleafhome{/tmp}
\newcommand{\inputplot}[2]{%
	\ifx\homepath\overleafhome%
	\IfBeginWith{#1}{plots}{\includegraphics{_main-figure#2.pdf}}{#1}%
	\else%
	{\sffamily\scriptsize\input{#1}}
\fi}

\newcommand{\newgroupwidth}[2]%
{\expandafter\xdef\csname groupwidth#1\endcsname{#2}}

\newcounter{groupwidth}
\newsavebox{\groupwidthbox}
\makeatletter
{\edef\groupnumber{#1}%
	\stepcounter{groupwidth}%
	\@ifundefined{groupwidth\thegroupwidth}{\pgfmathsetlengthmacro{\mywidth}{\linewidth/\groupnumber}}%
	{\expandafter\let\expandafter\mywidth\csname groupwidth\thegroupwidth\endcsname}%
	\begin{lrbox}{\groupwidthbox}%
		\tikzset{/pgfplots/width={\mywidth}}%
		\ignorespaces}%
	{\end{lrbox}%
	\usebox\groupwidthbox
	\pgfmathsetlengthmacro{\mywidth}{\mywidth + (\linewidth - \wd\groupwidthbox)/\groupnumber}
	\immediate\write\@auxout{\string\newgroupwidth{\thegroupwidth}{\mywidth}}}
\makeatother

\title{Epidemic Learning: Boosting Decentralized Learning with Randomized Communication}

\author{%
\makecell{
Martijn de Vos \thanks{Authors are listed in alphabetical order.}  \hspace{20pt}
Sadegh Farhadkhani \thanks{Corresponding author \textless sadegh.farhadkhani@epfl.ch\textgreater.}   \hspace{20pt}
Rachid Guerraoui \\
Anne-Marie Kermarrec \hspace{20pt}
Rafael Pires \hspace{20pt} 
Rishi Sharma }
\AND
\vspace{-0.5cm}\\
EPFL, Switzerland
}

\begin{document}

\maketitle

\begin{abstract}
We present \Ac{EL}, a simple yet powerful \ac{DL} algorithm that leverages changing communication topologies to achieve faster model convergence compared to conventional \ac{DL} approaches.
At each round of \ac{EL}, each node sends its model updates to a \emph{random sample} of $s$ other nodes (in a system of $n$ nodes).
We provide an extensive theoretical analysis of \ac{EL}, demonstrating that its changing topology culminates in superior 
convergence properties compared to the state-of-the-art (static and dynamic) topologies.
Considering smooth non-convex loss functions, the 
 number of  transient iterations for \ac{EL}, 
\ie, the rounds required to achieve asymptotic linear speedup, is in $\mathcal{O}(\nicefrac{n^3}{\samplenum^2})$ which outperforms the best-known bound $\mathcal{O}({n^3})$ by a factor of $ s^2 $, indicating the benefit of randomized communication for \ac{DL}.
We empirically evaluate \ac{EL}
in a 96-node network and compare its performance with state-of-the-art \ac{DL} approaches.
Our results illustrate that \ac{EL} converges up to $ 1.7\times $ quicker than baseline \ac{DL} algorithms and attains 2.2\% higher accuracy for the same communication volume.
\end{abstract}

\section{Introduction}

In Decentralized Learning (DL), multiple machines (or nodes) collaboratively train a machine learning model without any central server~\cite{lian2017can,lu2021optimal,nedic2020distributed}.
Periodically, each node updates the model using its local data, sends its model updates to other nodes, and averages the received model updates, all without sharing raw data.
Compared to centralized approaches~\cite{li2013parameter}, \ac{DL} circumvents the need for centralized control, ensures scalability~\cite{kairouz2021advances,lian2017can}, and avoids imposing substantial communication costs on a central server~\cite{ying2021bluefog}.
However, \ac{DL} comes with its own challenges. The exchange of model updates with all nodes can become prohibitively expensive in terms of communication costs as the network size grows~\cite{kong2021consensus}.
For this reason, nodes in \ac{DL} algorithms usually exchange model updates with only a small number of other nodes in a particular round, \ie they perform partial averaging instead of an All-Reduce (network-wide) averaging of local model updates~\cite{ying2021exponential}.

A key element of \ac{DL} algorithms is the communication topology, governing how model updates are exchanged between nodes.
The properties of the communication topology are critical for the performance of \ac{DL} approaches as it directly influences the speed of convergence~\cite{le2023refined,vogels2023beyond,wang2019matcha, scaman2018optimal}.
The seminal \dpsgd algorithm and many of its proposed variants rely on a static topology, \ie each node exchanges its model with a set of neighboring nodes that remain fixed throughout the training process~\cite{lian2017can}.
More recent approaches study changing topologies, \ie topologies that change during training, with notable examples being time-varying graphs~\cite{koloskova2020unified,lu2020decentralized,rogozin2022decentralized}, one-peer exponential graphs~\cite{ying2021exponential}, EquiTopo~\cite{song2022communicationefficient}, and Gossip Learning~\cite{hegedHus2019gossip,hegedHus2021decentralized,jelasity2007gossip}. We discuss these works in more detail in Section~\ref{sec:related_work}.

This paper investigates the benefits of \emph{randomized communication} for \ac{DL}.
Randomized communication, extensively studied in distributed computing, has been proven to enhance the performance of fundamental algorithms, including consensus and data dissemination  protocols~\cite{kempe2003gossip,boyd2006randomized,cason2021gossip}.
In the case of \ac{DL}, randomized communication can reduce the convergence time and therefore communication overhead~\cite{dhasade:2023:dcpy}. 
In this work, we specifically consider the setting where each node communicates with a random subset of other nodes that changes at each round, as with \emph{epidemic} interaction schemes~\cite{eugster2004epidemic,montresor2017gossip}.
We also focus on the scenario where data is unevenly distributed amongst nodes, \ie \ac{non-IID} settings, a common occurrence in  \ac{DL}~\cite{hsieh2020non}.

\begin{wrapfigure}{R}{0.4\textwidth}
	\centering
	\inputplot{plots/badstatic}{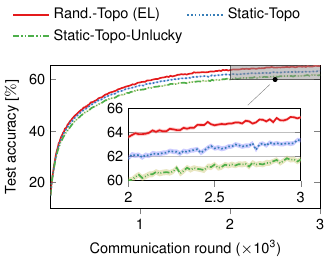}
	\caption{Randomized topologies can converge quicker than static ones.}
	\label{fig:motivational_fig}
\end{wrapfigure}

To illustrate the potential of randomized communication, we empirically compare model convergence in Figure~\ref{fig:motivational_fig} for static $s$-regular topologies (referred to as \texttt{Static-Topo}) and randomized topologies (referred to as \texttt{Rand.-Topo} (\Ac{EL}), our work) in 96-node networks on the \cifar learning task.
As particular static topologies can lead to sub-optimal convergence, we also experiment with the setting in which nodes are stuck in an \emph{unlucky} static $s$-regular topology (referred to as \texttt{Static-Topo-Unlucky}).
This unlucky topology consists of two network partitions, connected by just two edges.
Figure~\ref{fig:motivational_fig} reveals that dynamic topologies converge quicker than static topologies.
After \num{3000} communication rounds, \texttt{Rand.-Topo} (EL, our work) achieves 65.3\% top-1 test accuracy, compared to 63.4\% and 61.8\% for Static-Random and Static-Unlucky, respectively.
Additional experiments can be found in Section~\ref{sec:eval}, and their setup is elaborated in~\cref{app:exp}.

\paragraph{Our contributions} 
Our paper makes the following four contributions:
\begin{itemize}[leftmargin=*]
\item We formulate, design, analyze, and experimentally evaluate \textbf{Epidemic Learning (EL)}, a novel \ac{DL} algorithm in which nodes collaboratively train a machine learning model using a dynamically changing, randomized communication topology.
More specifically, in \ac{EL}, at each round, each node sends its model update to a \emph{random sample} of $s$ other nodes (out of $ n $ total nodes).
This process results in a randomized topology that changes every round.

\item %
We first analyze an \ac{EL} variant, named \textbf{\elo}, where the union of the random samples by all nodes forms an $s$-regular random graph in each round.
\elo ensures a perfectly balanced communication load among the participating nodes as each node sends and receives exactly $ s $ model updates every round.
Nevertheless, achieving an $s$-regular graph necessitates coordination among the nodes, which is not ideal in a decentralized setting.  %
To address this challenge, we also analyze another \ac{EL} variant, named \textbf{\ello}. In \ello, each node independently and locally draws a uniformly random sample of $\samplenum$ other nodes at each round and sends its model update to these nodes.
We demonstrate that \ello enjoys a comparable convergence guarantee as \elo without requiring any coordination among the nodes and in a fully decentralized manner. 

\item Our theoretical analysis in Section~\ref{sec:theory} shows that \sys surpasses the best-known static and randomized topologies in terms of convergence speed.
More precisely, we prove that \sys converges with the rate 
$\mathcal{O} \left(\nicefrac{1}{\sqrt{{nT}}} + \nicefrac{1}{\sqrt[3]{{\samplenum T^2}}} + \nicefrac{1}{T}\right),$
where 
$T$ is the number of learning rounds.
Similar to most state-of-the-art algorithms for decentralized optimization~\cite{li2019communication,koloskova2020unified} and centralized \sgd (\eg, with a parameter server)~\cite{li2014scaling}, our rate asymptotically achieves \textbf{linear speedup}, \ie, when $T$ is sufficiently large, the first term in the convergence rate $\mathcal{O}(\nicefrac{1}{\sqrt{nT}})$ becomes dominant and improves with respect to the number of nodes. 

Even though linear speedup is a very desirable property, %
\ac{DL} algorithms often require many more rounds to reach linear speedup compared to centralized \sgd due to the additional error (the second term in the above convergence rate) arising from partial averaging of the local updates. To capture this phenomenon and to compare different decentralized learning algorithms, previous works~\cite{chen2021accelerating,pu2021sharp,song2022communicationefficient,ying2021exponential} adopt the concept of \textbf{transient iterations} which are the number of rounds before a decentralized algorithm reaches its linear speedup stage, \ie when $T$ is relatively small such that the second term of the convergence rate dominates the first term.

We derive that \sys requires $\mathcal{O}(\nicefrac{n^3}{\samplenum^2})$ transient iterations, which improves upon the best known bound by a factor of $\samplenum^2$. We also show this result in Table~\ref{tab:tran2}. We note that while \sys matches the state-of-the-art bounds when $\samplenum \in \mathcal{O}(1)$, it offers additional flexibility over other methods through parameter $\samplenum$ that provably improves the theoretical convergence speed depending on the communication capabilities of the nodes.
For instance, when $\samplenum \in \mathcal{O}(\log n)$ as in Erdős–Rényi and EquiStatic topologies, the number of transient iterations for \sys reduces to $\mathcal{O}(\nicefrac{n^3}{\log^2n})$, outperforming other methods.
This improvement comes from the fact that the second term in our convergence rate is superior to the corresponding term $\mathcal{O}(\nicefrac{1}{\sqrt[3]{p^2T^2}})$ in the rate of \dpsgd, 
 where $p \in (0, 1]$ is the  spectral gap of the mixing matrix.
We expound more on this in~\Cref{sec:theory}.

\item We present in~\Cref{sec:eval} our experimental findings. \newtext{Using two standard image classification datasets,} we compare 
\elo and \ello against static regular graphs and the state-of-the-art EquiTopo topologies. We find that \elo and \ello converge faster than the baselines and save up to 1.7$\times$ communication volume to reach the highest accuracy of the most competitive baseline.

\end{itemize}

\renewcommand{\arraystretch}{1.4}
\begin{table}[!tb]
    \small
  \centering
  \caption{Comparison of \ac{EL} with state-of-the-art \ac{DL} approaches (grouped by topology family). We compare \elo and \ello to ring, torus, Erdős–Rényi, exponential and EquiTopo topologies.}
  \begin{threeparttable}
  \begin{tabular}{c||cccc}
    \textbf{Method} & {\bf Per-Iter Out Msgs.} & {\bf Transient Iterations}  & {\bf Topology} & {\bf Communication} \\
    \hline
    \hline
     {\bf Ring \cite{koloskova2019decentralized}}& 2 & $\mathcal{O}(n^{11})$ & static & undirected\\
    \hline
    {\bf Torus \cite{koloskova2019decentralized}} &  4 & $\mathcal{O}(n^7)$ & static & undirected  \\
    \hline
    {\bf E.-R. Rand \cite{nedic2018network}} &  $\mathcal{O}(\log n)$ & $\Tilde{\mathcal{O}}(n^3)$ & static & undirected  \\
    \hline
    {\bf Static Exp. \cite{ying2021exponential}}  & $ \log n $ & $\mathcal{O}(n^3 \log^4 n)$ & static & directed \\
    {\bf One-Peer Exp. \cite{ying2021exponential}}  & 1 & $\mathcal{O}(n^3 \log^4 n)$ & semi-dynamic\tnote{1} & directed \\
    \hline
    {\bf D-EquiStatic \cite{song2022communicationefficient}}  &  $ \log n$ & $\mathcal{O}(n^3)$ & static & directed  \\
    {\bf U-EquiStatic \cite{song2022communicationefficient}}  &  $ \log n$ & $\mathcal{O}(n^3)$ & static & undirected  \\
    {\bf OD-EquiDyn \cite{song2022communicationefficient}}  & 1 & $\mathcal{O}(n^3)$ & semi-dynamic\tnote{1} & directed  \\
    {\bf OU-EquiDyn \cite{song2022communicationefficient}}  & 1 & $\mathcal{O}(n^3)$ & semi-dynamic\tnote{1} & undirected  \\
    \hline
    {\cellcolor{olive!20} \bf \elo (ours)}  & \cellcolor{olive!20}  $\samplenum$  & \cellcolor{olive!20} $\mathcal{O}(\nicefrac{n^3}{\samplenum^2})$  & \cellcolor{olive!20} rand.-dynamic\tnote{2} & \cellcolor{olive!20}  undirected \\
    {\cellcolor{olive!20} \bf \ello (ours)}  & \cellcolor{olive!20}  $\samplenum$ & \cellcolor{olive!20} $\mathcal{O}(\nicefrac{n^3}{\samplenum^2})$  & \cellcolor{olive!20} rand.-dynamic\tnote{2} & \cellcolor{olive!20} directed  \\
  \end{tabular}%
  \begin{tablenotes}
    \item[1] Semi-dynamic topologies remain fixed throughout the learning process but nodes select subsets of adjacent (neighboring) nodes each round to communicate with. %
    \item[2] In a randomized-dynamic topology, the topology is replaced each round.
    \end{tablenotes}
  \end{threeparttable}
  \label{tab:tran2}
  \vspace{-4mm}
\end{table}

\section{Epidemic Learning}
\label{sec:epidemic_learning}
In this section, we first formally define the decentralized optimization problem. Then we outline our \sys algorithm and its variants in Section~\ref{sec:alg}.

\subsection{Problem statement}
We consider a system of $n$ nodes $[n]:=\{1,\ldots,n\}$ where the nodes can communicate by sending messages.
\newtext{Similar to existing work in this domain~\cite{song2022communicationefficient}, we consider settings in which a node can communicate with all other nodes.
The implications of this assumption are further discussed in~\Cref{sec:network_connectivity}.}
Consider a  data space $\dataspace$ and a loss function $\lossperpoint: \R^d \times \dataspace \to \R$. Given a parameter $\model{}{}\in \R^d$, a data point $\datapoint{}{} \in \dataspace$ incurs a loss of value $\lossperpoint(\model{}{}, \, \datapoint{}{})$. Each node $i \in [n]$ has a data distribution $\dist{i}$ over $\dataspace$, which may differ from the data distributions of other nodes.
We define the local loss function of $i$ over distribution $\dist{i}$ as ${\localloss{i}(\model{}{}):= \condexpect{\datapoint{}{}\sim \dist{i} }{\lossperpoint(\model{}{},\datapoint{}{})}}$.
The goal is to collaboratively minimize the {\em global average loss} by solving the following optimization problem:

 \begin{equation}
 \label{eq:optprblm}
    \min_{\model{}{}\in\R^d}\left[\loss(\model{}{}) := \frac{1}{n} \sum_{i \in [n]}{\localloss{i}(\model{}{})}\right].
\end{equation}

\begin{algorithm}[b!]
\small
\begin{algorithmic}[1]
\caption{Epidemic Learning as executed by a node $i$}
\label{algo}
\State \textbf{Require}: Initial model $\model{i}{0} =\model{}{0} \in \R^d$, number of rounds $T$, step-size $\localstep{},$ sample size $\samplenum$.
\For{$t=0,\dots, \, T-1$} \Comment{Line 3-5: Local training phase}
\State Randomly sample a data point $\datapoint{i}{t} $ from the local data distribution $\dist{i}$
\State Compute the stochastic gradient $\gradient{i}{t}:=\nabla\lossperpoint(\model{i}{t},\datapoint{i}{t})$
\State Partially update local model $\model{i}{t + 1/2} := \model{i}{t} - \localstep{} \, \gradient{i}{t}$ \Comment{Line 6-9: Random communication phase}
\State Sample $\samplenum$ other nodes from $[n] \setminus \{i\}$ using \elo or \ello 
\State Send $\model{i}{t + \nicefrac{1}{2}}$ to the selected nodes 
\State Wait for the set of updated models $\receivedsubset{i}{t}$ \Comment{$\receivedsubset{i}{t}$ is the set of received models by node $ i $ in round $ t $}
\State Update $\model{i}{t + 1}$ to the average of available updated models according to~(\ref{eq:aggregation})
\EndFor
\end{algorithmic}
\end{algorithm}

\subsection{Description of \sys}
\label{sec:alg}
We outline  \sys, executed by node $ i $, in Algorithm~\ref{algo}.
We define the initial model of node $ i $ as $\model{i}{0} $ and a step-size $\localstep{}$ used during the local model update.
The \sys algorithm runs for $ T $ rounds. Each round consists of two phases: a \emph{local update phase} (line 3-5) in which the local model is updated using the local dataset of node $ i $, and a \emph{random communication phase} (line 6-9) in which model updates are sent to other nodes chosen randomly.
In the local update phase, node $ i $ samples a data point $\datapoint{i}{t} $ from its local data distribution $\dist{i}$ (line 3), computes the stochastic gradient $\gradient{i}{t} $ (line 4) and partially updates its local model $\model{i}{t + 1/2} $ using step-size $ \gamma $ and gradient $\gradient{i}{t} $ (line 5).

The random communication phase follows, where node $i$ first selects $\samplenum$ of other nodes from the set of all nodes excluding itself: $[n] \setminus \{i\}$ (line 6).
This sampling step is the innovative element of \sys, and we present two variants later.
It then sends its recently updated local model $\model{i}{t + \nicefrac{1}{2}}$ to the selected nodes and waits for model updates from other nodes. Subsequently, each node $i$ updates its model based on the models it receives from other nodes according to~\Cref{eq:aggregation}. The set of nodes that send their models to node $i$ is denoted by $\receivedsubset{i}{t}$. The new model for node $i$ is computed as a weighted average of the models received from the other nodes and the local model of node $i$, where the weights are inversely proportional to the number of models received plus one.

\begin{equation}
\label{eq:aggregation}
    \model{i}{t+1} :=  \frac{1}{\card{\receivedsubset{i}{t}} + 1} \left(\model{i}{t + \nicefrac{1}{2} } + \sum_{j \in \receivedsubset{i}{t}} \model{j}{t + \nicefrac{1}{2}} \right).
\end{equation}

We now describe two approaches to sample $ s $ other nodes (line 6), namely \elo and \ello:

\begin{figure}
     \centering
     \includegraphics[width=\textwidth]{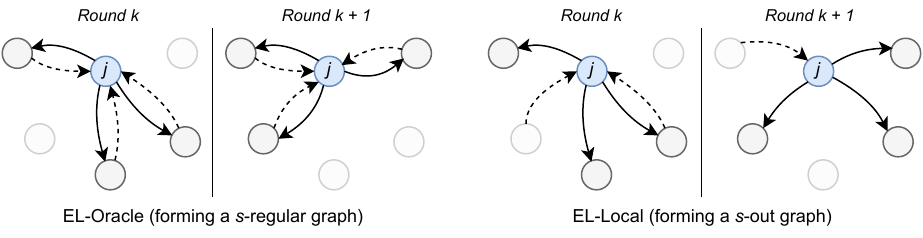}
    \caption{\elo (left) and \ello (right), from the perspective of node $ j $, with $ s = 3 $ and for two rounds. We show both outgoing model updates (solid line) and incoming ones (dashed line).}
    \label{fig:el}
    \vspace{-0.3cm}
\end{figure}

\paragraph{\elo} With \elo, the union of selected communication links forms a $s$-regular topology in which every pair of nodes has an equal probability of being neighbors.
Moreover, if node $ i $ samples node $ j $, $ j $ will also sample $ i $ (communication is undirected).
Figure~\ref{fig:el} (left) depicts \elo sampling from the perspective of node $j$ in two consecutive iterations.
One possible way to generate such a dynamic graph is by generating an $s$-regular structure and then distributing a random permutation of nodes at each round.
Our implementation (see Section~\ref{sec:eval}) uses a central coordinator to randomize and synchronize the communication topology each round.

\paragraph{\ello} Constructing the $ s$-regular topology in \elo every round can be challenging in a fully decentralized manner as it requires coordination amongst nodes to ensure all nodes have precisely $ s $ incoming and outgoing edges.
This motivates us to introduce \ello, a sampling approach where each node $ i $ locally and independently samples $\samplenum$ other nodes and  sends its model update to them, without these $ s $ nodes necessarily sending their model back to $ i $. The union of selected nodes now forms a $s$-out topology.
Figure~\ref{fig:el} (right) depicts \ello sampling from the perspective of node $j$ in two consecutive iterations.
In practice, \ello can be realized either by exchanging peer information before starting the learning process or using a decentralized peer-sampling service that provides nodes with (partial) views on the network~\cite{jelasity2007gossip,nedelec2018adaptive,voulgaris2005cyclon}.
\newtext{While both the topology construction in \elo as well as peer sampling in \ello add some communication and computation overhead, this overhead is minimal compared to the resources used for model exchange and training.}

\begin{wrapfigure}{R}{0.4\textwidth}
	\centering
	\inputplot{plots/distribution}{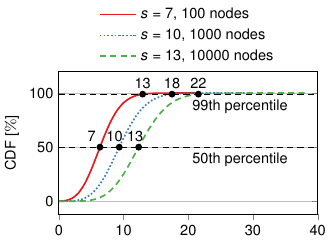}
	\caption{The distribution of incoming models, for different values of $ n $ and $ s $.}
	\label{fig:indegree}
\end{wrapfigure}

Even though each node sends $\samplenum$ messages for both \elo and \ello, in \ello, different nodes may receive different numbers of messages in each training round as each node selects the set of its out-neighbors locally and independent from other nodes.
While this might cause an imbalance in the load on individual nodes, we argue that this does not pose a significant issue in practice.
To motivate this statement, we run an experiment with $ n = $ \num{100}, \num{1000} and \num{10000} and for each value of $ n $ set $ s = \lceil log_2(n) \rceil $.
We simulate up to \num{5000} rounds for each configuration.
We show in Figure~\ref{fig:indegree} a CDF with the number of incoming models each round and observe that the distribution of the number of models received by each node is very light-tailed.
In a \num{10000} node network (and $s = $ \num{13}), nodes receive less than \num{22} models in \num{99}\% of all rounds.
As such, it is improbable that a node receives a disproportionally large number of models in a given round.
In practice, we can alleviate this imbalance issue by adopting a threshold value $k$ on the number of models processed by each node in a particular round and ignoring incoming models after having received $ k $ models already.
\newtext{The sender node can then retry model exchange with another random node that is less occupied.}

\section{Theoretical Analysis}
\label{sec:theory}

In this section, we first present our main theoretical result demonstrating the finite-time convergence of \sys. We then compare our result with the convergence rate of \dpsgd on different topologies.

\subsection{Convergence of \newalgorithm{}}

In our analysis, we consider the class of smooth loss functions, and we assume that the variance of the noise of stochastic gradients is bounded.
We use of the following assumptions that are classical to the analysis of stochastic first-order methods and hold for many learning problems~\cite{bottou2018optimization,ghadimi2013stochastic}.
\begin{assumption}[Smoothness]
\label{ass:smoothness}
    For all $i \in [n]$, the function $\localloss{i}:\R^d \rightarrow \R$ is differentiable and there exists $L < \infty$, such that for all $\model{}{}, \modelp{}{} \in \R^d$, 
    \begin{equation*}
       \norm{\nabla\localloss{i}(\modelp{}{})- \nabla\localloss{i}(\model{}{})}\leq L \norm{\modelp{}{}-\model{}{}}.
    \end{equation*}
\end{assumption}

\begin{assumption} [Bounded stochastic noise] 
There exists $\sigma < \infty$ such that for all $i \in [n]$, and $\model{}{}\in \R^d$,
\label{ass:stochastic_noise}
    \begin{align*}
        \condexpect{\datapoint{}{} \sim \dist{i}}{\norm{\nabla\lossperpoint(\model{}{},\datapoint{}{})-{\localloss{i}}(\model{}{})}^2} \leq \sigma^2.
    \end{align*}
\end{assumption}
Moreover, we assume that the heterogeneity among the local loss functions measured by the
average distance between the local gradients  is bounded.
\begin{assumption}[Bounded heterogeneity]
    \label{ass:bounded_heterogeneity}
    There exists $\heterparam < \infty$, such that for all $\model{}{}\in \R^d$,
    \begin{equation*}
        \frac{1}{n} \sum_{i \in [n]} {\norm{\nabla\localloss{i}(\model{}{})-\nabla\loss(\model{}{})}^2} \leq \heterparam^2.
    \end{equation*}
\end{assumption}
 We note that this assumption is standard in {\em heterogeneous} (a.k.a. non-i.i.d) settings, \ie when nodes have different data distributions~\cite{lian2017can,ying2021exponential}. In particular, $\heterparam$ can be bounded based on the closeness of the underlying local data distributions~\cite{fallah2020personalized}. We now present our main theorem.

\noindent \fcolorbox{black}{white}{
\parbox{0.97\textwidth}{\centering
\begin{theorem}
\label{th:main}
Consider Algorithm~\ref{algo}. 
Suppose that assumptions~\ref{ass:smoothness},~\ref{ass:stochastic_noise} and~\ref{ass:bounded_heterogeneity} hold true.  Let $\Delta_0$ be a real value such that $\loss(\model{}{0}) - \min_{\model{}{} \in \R^d} \loss(\model{}{}) \leq \Delta_0$.
Then, for any $T\geq 1$, $n\geq 2$, and $\samplenum\geq 1$:

\textbf{a)} For \textbf{\elo,} setting
{\small
$$\localstep \in \Theta\left(\min\left\{\sqrt{\frac{n\Delta_0}{TL\sigma^2}},\sqrt[3]{\frac{\Delta_0}{T \contractionp{\samplenum}L^2 \left( \sigma^2 + \heterparam^2\right)}},\frac{1}{L}\right\} \right),$$}
we have
\small
\begin{align*}
    \frac{1}{n}\sum_{i \in [n]}\frac{1}{T}\sum_{t = 0}^{T-1}\expect{ \norm{\nabla \loss \left( \model{i}{t}\right)}^2} 
     &\in \mathcal{O} \left(\sqrt{\frac{L\Delta_0\sigma^2}{nT}} + \sqrt[3]{\frac{\contractionp{\samplenum} L^2\Delta_0^2 \left( \sigma^2 + \heterparam^2\right)}{T^2}} + \frac{L\Delta_0}{T}\right),
\end{align*}
\normalsize
where 
\small
    \begin{align*}
        \contractionp{\samplenum} := \frac{1}{\samplenum + 1}\left( 1 - \frac{s}{n-1}\right) \in \mathcal{O}(\frac{1}{\samplenum}).
    \end{align*}
\normalsize
\textbf{b)} For \textbf{\ello}, setting 
{\small
$$\localstep \in \Theta\left(  \min\left\{\sqrt{\frac{n\Delta_0}{T \left(\sigma^2+ \contraction{\samplenum}\heterparam^2\right)L}},\sqrt[3]{\frac{\Delta_0}{ T\contraction{\samplenum} L^2\left( \sigma^2 + \heterparam^2\right)}},\frac{1}{L}\right\}\right),$$}
we have
\small
\begin{align*}
    \frac{1}{n}\sum_{i \in [n]}\frac{1}{T}\sum_{t = 0}^{T-1}\expect{ \norm{\nabla \loss \left( \model{i}{t}\right)}^2} 
     &\in \mathcal{O} \left(\sqrt{\frac{L\Delta_0(\sigma^2+ \contraction{\samplenum}\heterparam^2)}{nT}} + \sqrt[3]{\frac{\contraction{\samplenum} L^2\Delta_0^2 \left( \sigma^2 + \heterparam^2\right)}{T^2}} + \frac{L\Delta_0}{T}\right),
\end{align*}
\normalsize
where  \small
\begin{align*}
       \contraction{\samplenum} :=         
        \frac{1}{\samplenum}\left(1-\left(1-\frac{\samplenum}{n-1}\right)^n\right)-\frac{1}{n-1} \in \mathcal{O}(\frac{1}{\samplenum}).
    \end{align*}

\normalsize
\end{theorem}
}
}

To check the tightness of this result, we consider the special case when $\samplenum = n - 1$. Then, by Theorem~\ref{th:main}, we have $\contractionp{\samplenum} = \contraction{\samplenum} = 0$, and thus both of the convergence rates become $\mathcal{O} \left(\sqrt{\nicefrac{L\Delta_0\sigma^2}{nT}} +\nicefrac{L\Delta_0}{T}\right)$, which is the same as the convergence rate of (centralized) SGD for non-convex loss functions~\cite{ghadimi2013stochastic}. 
This is expected as, in this case, every node sends its updated model to all other nodes, corresponding to all-to-all communication in a fully-connected topology and thus perfectly averaging the stochastic gradients without any drift between the local models.

The proof of Theorem~\ref{th:main} is given in Appendix~\ref{sec:conv_proof}, where we obtain a tighter convergence rate than existing methods.
It is important to note that as the mixing matrix of a regular graph is doubly stochastic, for \elo, one can use the general analysis of \dpsgd with (time-varying) doubly stochastic matrices~\cite{koloskova2020unified,li2019communication} to obtain a convergence guarantee.
However,
 the obtained rate would not be as tight and would not capture the $\mathcal{O}(\nicefrac{1}{\sqrt[3]{\samplenum}})$ improvement in the second term, which is the main advantage of randomization (see~\Cref{sec:discussion_theory}).
Additionally, it is unclear how these analyses can be generalized to the case where the mixing matrix is not doubly stochastic, which is the case for our \ello algorithm.
Furthermore, another line of work~\cite{assran2019stochastic, nedic2014distributed,nedic2016stochastic} provides convergence guarantees for decentralized optimization algorithms based on the PushSum algorithm~\cite{kempe2003gossip}, with communicating the mixing weights.
It may be possible to leverage this proof technique  
to prove the convergence of \ello. However, this approach yields sub-optimal dimension-dependent convergence guarantees (\eg, see parameter $C$ in Lemma 3 of \cite{assran2019stochastic}) and does not capture the benefit of randomized communication.
\begin{remark}
    Most prior work~\cite{koloskova2020unified,lian2017can} provides the convergence guarantee on the average of the local models $\avgmodel{t} = \frac{1}{n} \sum_{i \in [n]} \model{i}{t}$. However, as nodes cannot access the global averaged model, we provide the convergence rate directly on the local models. Nonetheless, the same convergence guarantee as Theorem~\ref{th:main} also holds for the global averaged model. 
\end{remark}

\subsection{Discussion and comparison to prior results} 
\label{sec:discussion_theory}
To provide context for the above result, we note that the convergence rate of decentralized \sgd with non-convex loss functions and a doubly stochastic mixing matrix ~\cite{koloskova2020unified, li2019communication} is
\small
\begin{align}
\label{eq:rate_dsgd}
  \mathcal{O} \left(\sqrt{\frac{L\Delta_0\sigma^2}{nT}} + \sqrt[3]{\frac{ L^2\Delta_0^2 \left( p\sigma^2 + \heterparam^2\right)}{p^2T^2}} + \frac{L\Delta_0}{pT}\right),
\end{align}
\normalsize
where $p \in (0, 1]$ is  the  spectral gap of the mixing matrix and $\nicefrac{1}{p}$ is bounded by $\mathcal{O}(n^2)$ for ring~\cite{koloskova2019decentralized}, $\mathcal{O}(n)$ for torus~\cite{koloskova2019decentralized}, $\mathcal{O}(1)$ for Erdős–Rényi random graph~\cite{nedic2018network}, $\mathcal{O}(\log n)$ for exponential graph~\cite{ying2021exponential}, and $\mathcal{O}(1)$ for EquiTopo~\cite{song2022communicationefficient}.
We now compare the convergence of \sys against other topologies across two key properties: linear speed-up and transient iterations.

\paragraph{Linear speed-up} Both of our convergence rates preserve a linear speed-up of  $\mathcal{O}(\nicefrac{1}{\sqrt{nT}})$ in the first term. For \elo, this term is the same as~\eqref{eq:rate_dsgd}.
However, in the case of \ello, in addition to the stochastic noise $\sigma$, this term also depends on the heterogeneity parameter $\heterparam$ that vanishes when increasing the sample size $\samplenum$. This comes from the fact that, unlike \elo, the communication phase of \ello does not preserve the exact average of the local models (\ie $\sum_{i \in [n]} \model{i}{t+1} \neq \sum_{i \in [n]} \model{i}{t+\nicefrac{1}{2} }$), and it only preserves the average in expectation. This adds an error term to the rate of \ello. However, as the update vector remains an unbiased estimate of the average gradient, this additional term does not violate the linear speed-up property.
Our analysis suggests that setting $\samplenum \approx \frac{\heterparam^2}{\sigma^2}$ can help mitigate the effect of heterogeneity on the convergence of \ello.
Intuitively, more data heterogeneity leads to more disagreement between the nodes, which requires more communication rounds to converge.

\paragraph{Transient  iterations}
Our convergence rates offer superior second and third terms compared to those in~\eqref{eq:rate_dsgd}. This is because first, $p$ can take very small values, particularly when the topology connectivity is low (\eg, $\frac{1}{p} \in \mathcal{O}(n^2)$ for a ring) and second, even when the underlying topology is well-connected and $\frac{1}{p} \in \mathcal{O}(1)$, such as in EquiTopo~\cite{song2022communicationefficient}, the second term in our rates still outperforms the one in~\eqref{eq:rate_dsgd} by a factor of ${\sqrt[3]{s}}$. This improvement is reflected in the number of transient iterations before the linear speed-up stage, \ie the number of rounds required for the first term of the convergence rate to dominate the second term~\cite{ying2021exponential}.
In our rates, the number of transient iterations is in $\mathcal{O}(\nicefrac{n^3}{s^2})$, whereas in~\eqref{eq:rate_dsgd}, it is $\mathcal{O}(\nicefrac{n^3}{p^2})$ for the homogeneous case and $\mathcal{O}(\nicefrac{n^3}{p^4})$ for the heterogeneous case. We remark that $p \in (0,1]$, but $s \geq 1$ is an integer; therefore, when $\samplenum \in \mathcal{O}(1)$ the number of transient iterations for \ac{EL}  matches the state-of-the-art bound. However, it can be provably improved by increasing $\samplenum$ depending on the communication capabilities of the nodes, which adds more flexibility to \sys with theoretical guarantees compared to other methods. For instance, for $\samplenum \in \mathcal{O} (\log n)$ as in Erdős–Rényi and EquiStatic topologies, the number of transient iterations for \ac{EL} becomes $\mathcal{O}(n^3/\log^2 n)$ which outperforms other methods (also see Table~\ref{tab:tran2}).
 Crucially,  a key implication of this result is that our algorithm requires fewer rounds and, therefore, less communication to converge. We empirically show the savings in the communication of \ac{EL} in~\Cref{sec:eval}.

\section{Evaluation}
\label{sec:eval}
We present here the empirical evaluation of \sys and compare it with state-of-the-art \ac{DL} baselines.
We first describe the experimental setup and then show the performance of \elo and \ello.

\subsection{Experimental setup}
\paragraph{Network setup and implementation}
We deploy 96 \Ac{DL} nodes for each experiment, interconnected according to the evaluated topologies.
When experimenting with {\it s}-regular topologies, each node maintains a fixed degree of $\lceil log_2(n) \rceil$, \ie each node has 7 neighbors.
For \elo we introduce a centralized coordinator (oracle) that generates a \newtext{random} $7$-Regular topology at the start of each round and informs all nodes about their neighbors for the upcoming round.
For \ello we make each node aware of all other nodes at the start of the experiment. %
To remain consistent with other baselines, we fix $s = \lceil log_2(n) \rceil = 7$ when experimenting with \sys, \ie each node sends model updates to 7 other nodes each round.
Both \elo and \ello were implemented using the \texttt{DecentralizePy} framework~\cite{dhasade:2023:dcpy} and Python 3.8\footnote{Source code can be found at \href{https://github.com/sacs-epfl/decentralizepy/releases/tag/epidemic-neurips-2023}{https://github.com/sacs-epfl/decentralizepy/releases/tag/epidemic-neurips-2023}.}.
For reproducibility, a uniform seed was employed for all pseudo-random generators within each node.

\paragraph{Baselines}
We compare the performance of \elo and \ello against three variants of \dpsgd. %
Our first baseline is a fully-connected topology (referred to as \texttt{Fully connected}), which presents itself as the upper bound for performance given its optimal convergence rate~\cite{assran2019stochastic}.
We also compare with a {\it s}-regular static topology, the non-changing counterpart of \elo (referred to as \texttt{7-Regular static}).
\newtext{This topology is randomly generated at the start of each run according to the random seed, but is kept fixed during the learning.}
Finally, we compare \sys against the communication-efficient topology U-EquiStatic~\cite{song2022communicationefficient}.
Since U-EquiStatic topologies can only have even degrees, we generate U-EquiStatic topologies with a degree of 8 to ensure a fair comparison.
We refer to this setting as \texttt{8-U-EquiStatic}.

\paragraph{Learning task and partitioning}
We evaluate the baseline algorithms using the \cifar image classification dataset~\cite{krizhevsky2014cifar} \newtext{and the \femnist dataset, the latter being part of the LEAF benchmark~\cite{caldas2018leaf}.
In this section we focus on the results for \cifar and present the results for \femnist in~\cref{sec:exp_femnist}.}
We employ a \ac{non-IID} data partitioning using the Dirichlet distribution function~\cite{hsu2019measuring}, parameterized with $\alpha = 0.1$.
We use a \textsc{GN-LeNet} convolutional neural network~\cite{hsieh2020non}.
Full details on our experimental setup \newtext{and hyperparameter tuning} can be found in \newtext{Appendix~\ref{sec:exp_setup}}.

\paragraph{Metrics}
We measure the average top-1 test accuracy and test loss of the model on the test set in the \cifar learning task every 20 communication rounds.
Furthermore, we present the average top-1 test accuracy against the cumulative outgoing communication per node in bytes.
We also emphasize the number of communication rounds taken by \sys to reach the best top-1 test accuracy of static 7-Regular topology.
We run each experiment five times with different random seeds, and we present the average metrics with a 95\% confidence interval.

\begin{figure}[tb!]
    \vspace{-0.3cm}
	\centering
	\inputplot{plots/elcomparison-cifar}{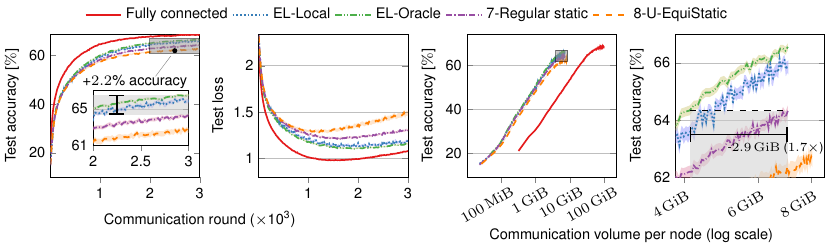}
	\caption{Communication rounds vs. top-1 test accuracy and (left) and communication volume per node vs. test accuracy (right) \newtext{for the \cifar dataset}.\rp{new data}}
	\label{fig:elcomparison}
    \vspace{-0.3cm}
\end{figure}

\subsection{\sys against baselines}
\Cref{fig:elcomparison} shows the performance of \elo and \ello against the baselines for the \cifar dataset.
\dpsgd over a fully-connected topology achieves the highest accuracy, as expected, but incurs more than an order of magnitude of additional communication.
\elo converges faster than its static counterparts of \texttt{7-Regular static} and \texttt{8-U-EquiStatic}.
After \num{3000} communication rounds, \elo and \ello achieve up to 2.2\% higher accuracy compared to \texttt{7-Regular static} (the most competitive baseline).
Moreover, \elo takes up to 1.7$\times$ fewer communication rounds and saves \SI{2.9}{\gibi\byte} of communication to reach the best accuracy attained by 7-Regular static.
Surprisingly, \texttt{8-U-EquiStatic} shows worse performance compared to \texttt{7-Regular static}.
The plots further highlight that the less constrained \ello variant has a very competitive performance compared to \elo: there is negligible utility loss when sampling locally compared to generating a $ s $-Regular graph every round.
We provide additional observations and results in~\cref{app:exp}.
\renewcommand{\arraystretch}{1.1}
\begin{table}[t!]
    \small
    \vspace{-0.2cm}
	\centering
	\caption{A summary of key experimental findings \newtext{for the \cifar dataset.}}
    \begin{threeparttable}
	\begin{tabular}{c c c c}
		\toprule
		\textbf{Topology} & \textbf{Top-1 Test Accuracy} & \textbf{Top-1 Test Loss} & \textbf{Communication to Target Accuracy} \\
            & (\%) & & (\SI{}{\gibi\byte})\\
		\midrule
        \textbf{Fully connected} & \num{68.67} & \num{0.98} & \num{30.54}\\
		\textbf{7-Regular static} & \num{64.32} & \num{1.21} & \num{7.03}\\
        \textbf{8-U-EquiStatic}\tnote{1} & \num{62.72} & \num{1.28} & - \\
        \textbf{\elo} & \num{66.56} & \num{1.10} & \num{4.12} \\
        \textbf{\ello} & \num{66.14} & \num{1.12} & \num{4.37} \\
		\bottomrule
	\end{tabular}
     \begin{tablenotes}
        \item[1] The 8-U-EquiStatic topology did not reach the 64.32\% target accuracy.
        \end{tablenotes}
    \end{threeparttable}
 \label{tab:experimental_results}
 \vspace{-0.2cm}
\end{table}

We summarize our main experimental findings in~\cref{tab:experimental_results}, which outlines the highest achieved top-1 test accuracy, lowest top-1 test loss, and communication cost to a particular target accuracy for our evaluated baselines.
This target accuracy is chosen as the best top-1 test accuracy achieved by the 7-Regular static topology (64.32\%), and the communication cost to reach this target accuracy is presented for all the topologies in the right-most column of~\cref{tab:experimental_results}.
In summary, \elo and \ello converge faster and to higher accuracies compared to 7-Regular static and 8-U-EquiStatic topologies, and require 1.7$\times$ and 1.6$\times$ less communication cost to reach the target accuracy, respectively.

\subsection{\newtext{Sensitivity Analysis of sample size $ s $}}

\begin{figure*}[t!]
	\centering
	\inputplot{plots/sensitivity_rebuttal.tex}{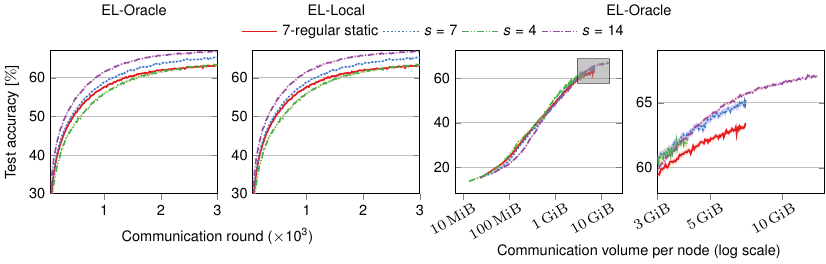}
	\caption{The test accuracy of \elo and \ello (left) and communication volume per node of \elo (right) for different value of sample size $s$ and a 7-Regular static topology.}
	\label{fig:sensitivity}
 \vspace{-0.3cm}
\end{figure*}

\newtext{The sample size $ s $ determines the number of outgoing neighbours of each node at each round of \elo and \ello.
We show the impact of this parameter on the test accuracy in the two left-most plots in~\cref{fig:sensitivity}, for varying values of $s$, against the baseline of a 7-regular static graph.
We chose the values of $s$ as $\lceil \frac{log_2(n)}{2} \rceil = 4$, and $\lceil 2log_2(n) \rceil = 14$, over and above $\lceil log_2(n) \rceil = 7$.
We observe that, as expected, increasing $s$ leads to quicker convergence for both \elo and \ello (see Theorem~\ref{th:main}).
Increasing $\samplenum$ also directly increases the communication volume.
The two right-most plots in~\cref{fig:sensitivity} show the test accuracy of \elo when the  communication volume increases.
After \num{3000} communication rounds, \elo with $ s = 14 $ has incurred a communication volume of \SI{15.1}{\gibi\byte}, compared to \SI{4.3}{\gibi\byte} for $ s = 4 $. 
An optimal value of $ s $ depends on the network of the environment where \ac{EL} is deployed. In data center settings where network links usually have high capacities, one can employ a high value of $\samplenum$. In edge settings with limited network capacities, however, the value of $\samplenum$ should be smaller to avoid network congestion.}

\section{Related Work}
\label{sec:related_work}

\paragraph{Decentralized Parallel Stochastic Gradient Descent (D-PSGD)}
Stochastic Gradient Descent is a stochastic variant of the gradient descent algorithm and is widely used to solve optimization problems at scale~\cite{ghadimi2013stochastic,moulines2011non,nemirovski2009robust}.
Decentralized algorithms using SGD have gained significant adoption as a method to train machine learning models, with Decentralized Parallel SGD (D-PSGD) being the most well-known \ac{DL} algorithm~\cite{ghadimi2016mini,lian2015asynchronous,lian2017can,lian2018asynchronous,patarasuk2009bandwidth}.
D-PSGD avoids a server by relying on a communication topology describing how peers exchange their model updates~\cite{lian2017can}.

\paragraph{Static topologies}
In a static topology, all nodes and edges remain fixed throughout the training process.
The convergence speed of decentralized optimization algorithms closely depends on the mixing properties of the underlying topologies. There is a large body of work~\cite{benjamini2014mixing,beveridge2016best,boyd2005mixing,nachmias2008critical} studying the mixing time of different random graphs  such as the Erdos-Renyi graph and the geometric random graph.
As the Erdos-Renyi graphs have better mixing properties~\cite{nedic2018network}, we compare \ac{EL} with this family in Table~\ref{tab:tran2}.
In \cite{ying2021exponential}, the authors prove that static exponential graphs in which nodes are connected to $ \mathcal{O}(\log n) $ neighbors are an effective topology for \ac{DL}.
EquiStatic is a static topology family whose consensus rate is independent of the network size~\cite{song2022communicationefficient}.

\paragraph{Semi-dynamic topologies} 
Several \ac{DL} algorithms impose a fixed communication topology at the start of the training but have nodes communicate with random subsets of their neighbors each round.
We classify such topologies as semi-dynamic topologies.
The one-peer exponential graph has each node cycling through their $ \mathcal{O}(\log n) $ neighbors and has a similar convergence rate to static exponential graphs~\cite{ying2021exponential}.
In the AD-PSGD algorithm, each node independently trains and averages its model with the model of a randomly selected neighbour~\cite{lian2018asynchronous}.
EquiDyn is a dynamic topology family whose consensus rate is independent of the network size~\cite{song2022communicationefficient}.

\paragraph{Time-varying and randomized topologies}
A time-varying topology is a topology that changes throughout the training process~\cite{casteigts2012time}.
The convergence properties of time-varying graphs in distributed optimization have been studied by various works~\cite{koloskova2019decentralized, lobel2010distributed,lu2020decentralized,nedic2014distributed,nedic2018network}. While these works provide convergence guarantees for decentralized optimization algorithms over time-varying (or random) topologies, they do not show the superiority of randomized communication, and they do not prove a convergence rate that cannot be obtained with a static graph~\cite{song2022communicationefficient}. Another work~\cite{liu2022decentralized} considers client subsampling in decentralized optimization where at each round, only a fraction of nodes participate in the learning procedure.
This approach is orthogonal to the problem we consider in this paper.

\newtext{\paragraph{Gossip Learning}}
Closely related to \ello is \ac{GL}, a \ac{DL} algorithm in which each node progresses through rounds independently from other peers~\cite{ormandi2013gossip}.
In each round, a node sends their model to another random node and aggregates incoming models received by other nodes, weighted by age.
While \ac{GL} shows competitive performance compared to centralized approaches~\cite{hegedHus2019gossip,hegedHus2021decentralized}, \newtext{its convergence on non-convex loss functions has not been theoretically proven yet~\cite{onoszko2021decentralized}.}
\newtext{While at a high level, \ello with $ s = 1$ may look very similar to GL, there are some subtle differences.
First, GL operates in an asynchronous manner whereas \ac{EL} proceeds in synchronous rounds.
Second, GL applies weighted averaging when aggregating models, based on model age, and \ac{EL} aggregates models unweighted.
Third, if a node in \ac{GL} receives multiple models in a round, this node will aggregate the received model with its local model for each received model separately, whereas in \ac{EL}, there will be a single model update per round, and all the received models from that round are aggregated together.}

In contrast to the existing works, \ac{EL} leverages a fully dynamic and random topology that changes each round.
While static and semi-dynamic topologies have shown to be effective in certain settings, \ac{EL} surpasses their performance by enabling faster convergence, both in theory and in practice. %

\section{Conclusions}
We have introduced Epidemic Learning (EL), a new DL algorithm that accelerates model convergence and saves communication costs by leveraging randomized communication.
The key idea of \ac{EL} is that in each round, each node samples $ s $ other nodes in a $n$-node network and sends their model updates to these sampled nodes.
We introduced two variants of the sampling approach: \elo in which the communication topology forms a $s$-regular graph each round, and \ello which forms a $s$-out graph.
We theoretically proved the convergence of both \ac{EL} variants and derived that the number of transient iterations for \ac{EL} is in $\mathcal{O}(\nicefrac{n^3}{\samplenum^2})$, outperforming the best-known bound $\mathcal{O}({n^3})$ by a factor of $ s^2 $ for smooth non-convex functions.
Our experimental evaluation on the \cifar learning task showcases the 1.7$\times$ quicker convergence of \ac{EL} compared to baselines and highlights the reduction in communication costs.

\section*{Acknowledgement}
This work has been supported in part by the Swiss National
Science Foundation (SNSF) project 200021-200477.

\bibliographystyle{plainnat}
\bibliography{bibliography}
\newpage

\appendix
\newpage
\newpage
\begin{center}
    \LARGE \bf {Appendix}
\end{center}

\section*{Organization}

The appendices are organized as follows:
\begin{itemize}
    \item Appendix~\ref{sec:conv_proof} proves the convergence guarantee of \sys.
    \item In Appendix~\ref{sec:lemma_proof}, we prove the key lemmas that are used in the convergence proof.
    \item Appendix~\ref{app:exp} provides experimental details, computational resources used, and further evaluation of \sys with \iid data distributions and the \femnist dataset.
    \item In~\cref{sec:network_connectivity}, we discuss the effect of network connectivity on the performance of \ac{EL}.
\end{itemize}

\TODO{Update when almost done with the CR version!}

\section{Convergence Proof}
\label{sec:conv_proof}
In this section, we prove Theorem~\ref{th:main}, by setting
\begin{align}
    \localstep :=  \min\left\{\sqrt{\frac{n\Delta_0}{TL\sigma^2}},\sqrt[3]{\frac{\Delta_0}{100T \contractionp{\samplenum}L^2 \left( \sigma^2 + \heterparam^2\right)}},\frac{1}{20L}\right\},
\end{align}
for \elo,
and
\begin{align}
    \localstep := \min\left\{\sqrt{\frac{n\Delta_0}{T \left(211\sigma^2+332 \contraction{\samplenum}\heterparam^2\right)L}},\sqrt[3]{\frac{\Delta_0}{250 T\contraction{\samplenum} L^2\left( \sigma^2 + \heterparam^2\right)}},\frac{1}{20L}\right\},
\end{align}
for \ello.

{\bf Notation:} For any set of $n$ vectors $\{x^{(1)},\ldots,x^{(n)}\}$, we denote their average by $\bar{x} := \frac{1}{n}\sum_{i \in [n]}x^{(i)}$.

\subsection{Proof steps} 
We outline here the critical elements for proving Theorem~\ref{th:main}.

\paragraph{Mixing efficiency of \sys.}
First, we analyze the mixing properties of the communication phase of both variants of the algorithm, and we observe that:
\begin{enumerate}[label=(\alph*)]
     \item while \elo preserves the exact average of the local vectors, i.e., $\avgmodel{t + 1} = \avgmodel{t + \nicefrac{1}{2}}$, \ello
only does so in expectation. This property is critical for ensuring that the global update is an unbiased estimator of the average of local gradients, which is necessary for obtaining convergence guarantees with linear speed-up.
\item The communication phase of both \ello and \elo reduce the drift among the local models by a factor of  $\mathcal{O}(\nicefrac{1}{s})$.
\item  The variance of the averaged model after the communication phase, in \ello is in $\mathcal{O}(\nicefrac{1}{ns})$.
\end{enumerate}
More formally,  we have the following lemma.

\noindent \fcolorbox{black}{white}{
\parbox{0.97\textwidth}{\centering
\begin{restatable}{lemma}{lemcontraction}
\label{lem:contraction}
    Consider Algorithm~\ref{algo}. Let $n \geq 2$, $s\geq 1$, $T\geq 1$, and $t \in \{0,\ldots,T-1 \}$.
\begin{enumerate}
\item For \textbf{\elo}, we have
\begin{enumerate}
        \item ${\avgmodel{t+1}} = {\avgmodel{t+\nicefrac{1}{2}}}$, \label{dadsf}
        \item $\frac{1}{n^2} \sum_{i,j \in [n]} \expect{\norm{\model{i}{t+1} - \model{j}{t+1}}^2} \leq {\contractionp{\samplenum}} \cdot  \frac{1}{n^2} \sum_{i,j \in [n]} \expect{ \norm{\model{i}{t+\nicefrac{1}{2}} - \model{j}{t+\nicefrac{1}{2}}}^2} $,
    \end{enumerate}
     where 
    \begin{align*}
        \contractionp{\samplenum} := \frac{1}{\samplenum + 1}\left( 1 - \frac{s}{n-1}\right) \in \mathcal{O}(\frac{1}{\samplenum}).
    \end{align*}

\item
For \textbf{\ello}, we have:

    \begin{enumerate}
        \item \label{elo:avg} $\expect{\avgmodel{t+1}} = \expect{\avgmodel{t+\nicefrac{1}{2}}}$ (Note that we do not necessarily have ${\avgmodel{t+1}} = {\avgmodel{t+\nicefrac{1}{2}}}$),
        \item $\frac{1}{n^2} \sum_{i,j \in [n]} \expect{\norm{\model{i}{t+1} - \model{j}{t+1}}^2} \leq {\contraction{\samplenum}} \cdot  \frac{1}{n^2} \sum_{i,j \in [n]} \expect{ \norm{\model{i}{t+\nicefrac{1}{2}} - \model{j}{t+\nicefrac{1}{2}}}^2} $,
        \item \label{eq:propt} $\expect{\norm{\avgmodel{t+1}-\avgmodel{t+\nicefrac{1}{2}}}^2} \leq \frac{\contraction{\samplenum}}{2n} \cdot   \frac{1}{n^2} \sum_{i,j \in [n]} \expect{ \norm{\model{i}{t+\nicefrac{1}{2}} - \model{j}{t+\nicefrac{1}{2}}}^2} $,
    \end{enumerate}
    where 
    \begin{align*}
        \contraction{\samplenum} := 
        \frac{1}{\samplenum}\left(1-\left(1-\frac{\samplenum}{n-1}\right)^n\right)-\frac{1}{n-1} \in \mathcal{O}(\frac{1}{\samplenum})
    \end{align*}
\end{enumerate}
\end{restatable}
}
}

\paragraph{Uniform bound on model drift and gradient drift.}Next, using the previous lemma, we prove that, for a careful choice of the step-size $\localstep$, the local models stay close to each other during the learning process.

\noindent \fcolorbox{black}{white}{
\parbox{0.97\textwidth}{\centering
\begin{restatable}{lemma}{lemSGDdrift}
\label{lem:SGD:drift}
    Suppose that assumptions~\ref{ass:smoothness},~\ref{ass:stochastic_noise}, and \ref{ass:bounded_heterogeneity} hold true. Consider Algorithm~\ref{algo}. Consider a step-size $\localstep$ such that $\localstep  \leq \frac{1}{20L}$. For any $t\geq0$, we obtain that
\begin{align*}
    \frac{1}{n^2}\sum_{i,j\in[n]}\expect{\norm{\model{i}{t} - \model{j}{t} }^2} 
    & \leq 20\frac{1+3\contractionb{\samplenum}}{(1-\contractionb{\samplenum})^2} \contractionb{\samplenum}\localstep^2 \left( \sigma^2 + \heterparam^2\right),
\end{align*}
and
\begin{align*}
    \frac{1}{n^2}\sum_{i,j \in [n]}\expect{\norm{ \gradient{i}{t}-\gradient{j}{t}}^2}\leq 15\left(   \sigma^2 +  \heterparam^2 \right),
\end{align*}
 where $\contractionb{\samplenum} = \contractionp{\samplenum}$ for \elo and $\contractionb{\samplenum} = \contraction{\samplenum}$  for \ello as defined in Lemma~\ref{lem:contraction}.
\end{restatable}
}
}

\paragraph{Bound on the gradient norm.} Next, we obtain a bound on the gradient of the loss function $\loss(\model{}{})$ by analyzing the growth of $\loss(\model{}{})$, over the trajectory of the average of local models in Algorithm~\ref{algo}.

\noindent \fcolorbox{black}{white}{
\parbox{0.97\textwidth}{\centering
\begin{restatable}{lemma}{lemgrowthSGD}
\label{lem:growth_SGD}
Suppose that assumptions~\ref{ass:smoothness} and~\ref{ass:stochastic_noise} hold true. Consider Algorithm~\ref{algo} 
with $\localstep \leq \frac{1}{2L}$. For any $t \in \{0,\ldots,T-1\}$, we obtain that
\begin{align}\nonumber
    \expect{\norm{\nabla \loss \left( \avgmodel{t} \right)}^2}&\leq \frac{2}{\localstep}\expect{\loss(\avgmodel{t})- \loss(\avgmodel{t+1})} +\frac{L^2}{2n^2} \sum_{i,j \in [n]} \expect{\norm{\model{i}{t} - \model{j}{t}}^2}\\ 
     & + {2L\localstep \frac{\sigma^2}{n}} + \frac{2L}{\localstep}{\expect{\norm{\avgmodel{t+1} -\avgmodel{t+\nicefrac{1}{2}}}^2}}.\nonumber
\end{align}
\end{restatable}
}
}

\subsection{Proof of Theorem~\ref{th:main}}
We can now prove the theorem using the above lemmas.

\begin{proof}[Proof of Theorem~\ref{th:main}]
By Young's inequality, for any $i \in [n]$, we have 
\begin{align*}
    \expect{ \norm{\nabla \loss \left( \model{i}{t} \right)}^2} &\leq 2\expect{ \norm{\nabla \loss \left( \avgmodel{t} \right)}^2}  + 2\expect{ \norm{\nabla \loss \left( \avgmodel{t} \right)-\nabla \loss \left( \model{i}{t} \right)}^2}\\&\leq 2\expect{ \norm{\nabla \loss \left( \avgmodel{t} \right)}^2}  + 2L^2\expect{ \norm{\avgmodel{t} - \model{i}{t} }^2},
\end{align*}
where in the second inequality we used Assumption~\ref{ass:smoothness}.
Then, averaging over $i \in [n]$, we obtain that 
\begin{align*}
    \frac{1}{n} \sum_{i \in [n]}\expect{ \norm{\nabla \loss \left( \model{i}{t} \right)}^2} &\leq 2\expect{ \norm{\nabla \loss \left( \avgmodel{t} \right)}^2}  + 2L^2 \frac{1}{n} \sum_{i \in [n]}\expect{ \norm{\avgmodel{t} - \model{i}{t} }^2}\\
    &=  2\expect{ \norm{\nabla \loss \left( \avgmodel{t} \right)}^2}  + L^2 \frac{1}{n^2} \sum_{i,j \in [n]}\expect{ \norm{\model{i}{t}- \model{j}{t} }^2},
\end{align*}
where we used Lemma~\ref{lem:diam_equal}.
Combining this with Lemma~\ref{lem:growth_SGD}, we obtain that\footnote{In order to obtain a bound on $\expect{ \norm{\nabla \loss \left(\avgmodel{t}\right)}^2}$, we can skip this step and directly use Lemma~\ref{lem:growth_SGD}.}
\begin{align}\nonumber
        \frac{1}{n} \sum_{i \in [n]}\expect{ \norm{\nabla \loss \left( \model{i}{t} \right)}^2} &\leq \frac{4}{\localstep}\expect{\loss(\avgmodel{t})- \loss(\avgmodel{t+1})} +\frac{2L^2}{n^2} \sum_{i,j \in [n]} \expect{\norm{\model{i}{t} - \model{j}{t}}^2}\\ 
     & + {4L\localstep \frac{\sigma^2}{n}} + \frac{4L}{\localstep}{\expect{\norm{\avgmodel{t+1} -\avgmodel{t+\nicefrac{1}{2}}}^2}}.\label{eq:mainmain}
\end{align}

We now analyze the two variants of the algorithm separately.

\textbf{\elo:}\\
In this case, by Property~\eqref{elo:avg} of Lemma~\ref{lem:contraction}, we have ${\avgmodel{t+1}} - {\avgmodel{t+\nicefrac{1}{2}}} = 0$, therefore, by \eqref{eq:mainmain}, we have
\begin{align}\label{eq:maione}
       \frac{1}{n} \sum_{i \in [n]}\expect{ \norm{\nabla \loss \left( \model{i}{t} \right)}^2} &\leq \frac{4}{\localstep}\expect{\loss(\avgmodel{t})- \loss(\avgmodel{t+1})} +\frac{2L^2}{n^2} \sum_{i,j \in [n]} \expect{\norm{\model{i}{t} - \model{j}{t}}^2} + {4L\localstep \frac{\sigma^2}{n}} 
\end{align}
Note also that by Lemma~\ref{lem:contraction}, we have
\begin{align*}
     \frac{1}{n^2}\sum_{i,j\in[n]}\expect{\norm{\model{i}{t} - \model{j}{t} }^2} 
    & \leq 20\frac{1+3\contractionp{\samplenum}}{(1-\contractionp{\samplenum})^2} \contractionp{\samplenum}\localstep^2 \left( \sigma^2 + \heterparam^2\right),
\end{align*}
Moreover, as shown in Remark~\ref{rem:triv}, we have $\contractionp{\samplenum}\leq\nicefrac{1}{2}$, and thus, $20\frac{1+3\contractionp{\samplenum}}{(1-\contractionp{\samplenum})^2} \leq 200$. Therefore,
\begin{align*}
     \frac{1}{n^2}\sum_{i,j\in[n]}\expect{\norm{\model{i}{t} - \model{j}{t} }^2} 
    & \leq 200\contractionp{\samplenum}\localstep^2 \left( \sigma^2 + \heterparam^2\right),
\end{align*}
Combining this with \eqref{eq:maione}, we obtain that
\begin{align*}
    \frac{1}{n} \sum_{i \in [n]}\expect{ \norm{\nabla \loss \left( \model{i}{t} \right)}^2} &\leq \frac{4}{\localstep}\expect{\loss(\avgmodel{t})- \loss(\avgmodel{t+1})} + 400\contractionp{\samplenum}L^2\localstep^2 \left( \sigma^2 + \heterparam^2\right) + {4L\localstep \frac{\sigma^2}{n}} 
\end{align*}
Averaging over $t \in \{0,\ldots,T-1\}$, we obtain that
\begin{align}\nonumber
    \frac{1}{nT} \sum_{i \in [n]}\sum_{t=0}^{T-1}\expect{ \norm{\nabla \loss \left( \model{i}{t} \right)}^2} &\leq \frac{4}{\localstep T}\expect{\loss(\avgmodel{0})- \loss(\avgmodel{T})} + 400\contractionp{\samplenum}L^2\localstep^2 \left( \sigma^2 + \heterparam^2\right) + {4L\localstep \frac{\sigma^2}{n}} \\
    &\leq  \frac{4}{\localstep T}\Delta_0+ 400\contractionp{\samplenum}L^2\localstep^2 \left( \sigma^2 + \heterparam^2\right) + {4L\localstep \frac{\sigma^2}{n}},\label{eq:steppthree}
\end{align}
where we used the fact that $\loss(\avgmodel{0})- \loss(\avgmodel{T}) \leq \loss(\avgmodel{0})- \min_{\model{}{}\in \R^d} \loss(\model{}{})\leq \Delta_0$.
Setting 
\begin{align}\label{eq:steppone}
    \localstep := \min\left\{\sqrt{\frac{n\Delta_0}{TL\sigma^2}},\sqrt[3]{\frac{\Delta_0}{100T \contractionp{\samplenum}L^2 \left( \sigma^2 + \heterparam^2\right)}},\frac{1}{20L}\right\},
\end{align}
as $\frac{1}{\min\{a,b\}}= \max\{\frac{1}{a},\frac{1}{b}\}$ for any $a,b\geq0$, we have
\begin{align}\label{eq:setpptwo}
    \frac{1}{\localstep} = \max\left\{\sqrt{\frac{TL\sigma^2}{n\Delta_0}},\sqrt[3]{\frac{100T \contractionp{\samplenum}L^2 \left( \sigma^2 + \heterparam^2\right)}{\Delta_0}},20L\right\}\leq \sqrt{\frac{TL\sigma^2}{n\Delta_0}}+\sqrt[3]{\frac{100T \contractionp{\samplenum}L^2 \left( \sigma^2 + \heterparam^2\right)}{\Delta_0}}+20L,
\end{align}
Plugging \eqref{eq:steppone} and \eqref{eq:setpptwo} in \eqref{eq:steppthree}, we obtain that
\begin{align*}
    \frac{1}{nT} \sum_{i \in [n]}\sum_{t=0}^{T-1}\expect{ \norm{\nabla \loss \left( \model{i}{t} \right)}^2}
    &\leq 8\sqrt{\frac{L \Delta_0 \sigma^2}{nT}} + 38 \sqrt[3]{\frac{\contractionp{\samplenum}L^2\Delta_0^2 \left( \sigma^2 + \heterparam^2\right)}{T^2}} + { \frac{80L\Delta_0}{T}}\\
    & \in \mathcal{O} \left(\sqrt{\frac{L\Delta_0\sigma^2}{nT}} + \sqrt[3]{\frac{\contractionp{\samplenum} L^2\Delta_0^2 \left( \sigma^2 + \heterparam^2\right)}{T^2}} + \frac{L\Delta_0}{T}\right)
\end{align*}
where we used the fact that $\frac{800}{100^{(2/3)}}\leq38$.

\textbf{\ello:}\\
The proof is similar to \elo, up to some additional error terms. By Property \eqref{eq:propt} of Lemma~\ref{lem:contraction}, we obtain that
\begin{align}\nonumber
    \expect{\norm{\avgmodel{t+1} -\avgmodel{t+\nicefrac{1}{2}}}^2} 
    &\leq  \frac{\contraction{\samplenum}}{2n} \cdot \frac{1}{n^2} \sum_{i,j \in [n]} \expect{\norm{\model{i}{t+\nicefrac{1}{2}} -\model{j}{t+\nicefrac{1}{2}}}^2} \\
    &=  \frac{\contraction{\samplenum}}{2n} \cdot \frac{1}{n^2} \sum_{i,j \in [n]} \expect{\norm{ \model{i}{t} - \localstep{}  \gradient{i}{t} -\model{j}{t} + \localstep{}  \gradient{j}{t} }^2} \\
    &\leq \frac{\contraction{\samplenum}}{n}\cdot\frac{1}{n^2} \sum_{i,j \in [n]} \expect{\norm{\model{i}{t} -\model{j}{t}}^2} +\frac{ \localstep^2\contraction{\samplenum}}{n}\cdot\frac{1}{n^2} \sum_{i,j \in [n]} \expect{\norm{\gradient{i}{t} -\gradient{j}{t}}^2} \label{eq:diam_bound_growth},
\end{align}
\color{black}
where in the last inequality we use Young's inequality.
Combining this with~\eqref{eq:mainmain}, we obtain that 
\begin{align}\nonumber
    &\frac{1}{n} \sum_{i \in [n]}\expect{ \norm{\nabla \loss \left( \model{i}{t} \right)}^2} \leq \frac{4}{\localstep}\expect{\loss(\avgmodel{t})- \loss(\avgmodel{t+1})} +\frac{2L^2}{n^2} \sum_{i,j \in [n]} \expect{\norm{\model{i}{t} - \model{j}{t}}^2}\\  \nonumber
     & + {4L\localstep \frac{\sigma^2}{n}} + \frac{4L}{\localstep}\frac{\contraction{\samplenum}}{n}\cdot\frac{1}{n^2} \sum_{i,j \in [n]} \expect{\norm{\model{i}{t} -\model{j}{t}}^2} +\frac{4L \localstep\contraction{\samplenum}}{n}\cdot\frac{1}{n^2} \sum_{i,j \in [n]} \expect{\norm{\gradient{i}{t} -\gradient{j}{t}}^2}\\ \nonumber
     &= \frac{4}{\localstep}\expect{\loss(\avgmodel{t})- \loss(\avgmodel{t+1})} +(2L^2+\frac{4L\contraction{\samplenum}}{n\localstep})\frac{1}{n^2} \sum_{i,j \in [n]} \expect{\norm{\model{i}{t} - \model{j}{t}}^2}\\ 
     & + {4L\localstep \frac{\sigma^2}{n}} +\frac{4L \localstep\contraction{\samplenum}}{n}\cdot\frac{1}{n^2} \sum_{i,j \in [n]} \expect{\norm{\gradient{i}{t} -\gradient{j}{t}}^2} \label{eq:nrm}
\end{align}

Note that by Remark~\ref{rem:triv}, for any $\samplenum\geq 1$, and $n\geq 2$ we have $\contraction{\samplenum} \leq 1- \frac{1}{e}$, where $e$ is Euler's number. Therefore, 
\[
 20\frac{1+3\contraction{\samplenum}}{(1-\contraction{\samplenum})^2}\leq 500.
\]
Using this, the first bound in Lemma~\ref{lem:SGD:drift} can be simplified to 
\begin{align}
\label{eq:simplifiedbound}
    \frac{1}{n^2}\sum_{i,j\in[n]}\expect{\norm{\model{i}{t} - \model{j}{t} }^2} 
    & \leq 500 \contraction{\samplenum}\localstep^2 \left( \sigma^2 + \heterparam^2\right)
\end{align}

    Combining this with~\eqref{eq:nrm}, and the second bound of Lemma~\ref{lem:SGD:drift}, we have

\begin{align*}
    \frac{1}{n} \sum_{i \in [n]}\expect{ \norm{\nabla \loss \left( \model{i}{t}\right)}^2} &\leq \frac{4}{\localstep} \expect{\loss(\avgmodel{t})- \loss(\avgmodel{t+1})} + {4L\localstep \frac{\sigma^2}{n}} \left(1 + 15 \contraction{\samplenum} + 500\contraction{\samplenum}^2 \right)\\
    &+{4L\localstep \frac{\heterparam^2}{n}} \left( 15 \contraction{\samplenum} + 500\contraction{\samplenum}^2 \right) 
     + L^2 1000\contraction{\samplenum}\localstep^2 \left( \sigma^2 + \heterparam^2\right).
\end{align*}
Taking the average over $t \in \{0,\ldots,T-1\}$, we obtain that 
\begin{align*}
    \frac{1}{nT}\sum_{t = 0}^{T-1} \sum_{i \in [n]}\expect{ \norm{\nabla \loss \left( \model{i}{t}\right)}^2} &\leq \frac{4}{T\localstep} \Delta_0  + {4L\localstep \frac{\sigma^2}{n}} \left(1 + 15 \contraction{\samplenum} + 500\contraction{\samplenum}^2 \right)\\
    &+{4L\localstep \frac{\heterparam^2}{n}} \left( 15 \contraction{\samplenum} + 500\contraction{\samplenum}^2 \right) 
     + L^2 1000\contraction{\samplenum}\localstep^2 \left( \sigma^2 + \heterparam^2\right).
\end{align*}

Noting that $\contraction {\samplenum}< 1 -\frac{1}{e}$ (Remark~\ref{rem:triv}), we have
\begin{align}\label{eq:srrr}
    \frac{1}{nT}\sum_{t = 0}^{T-1} \sum_{i \in [n]}\expect{ \norm{\nabla \loss \left( \model{i}{t}\right)}^2} &\leq \frac{4}{T\localstep} \Delta_0 + {4\frac{ L\localstep}{n} \left(211\sigma^2+ 332\contraction{\samplenum}\heterparam^2\right)}
     + 1000\contraction{\samplenum} L^2\localstep^2 \left( \sigma^2 + \heterparam^2\right),
\end{align}
Now, setting
\begin{equation}\label{sttt}
    \localstep = \min\left\{\sqrt{\frac{n\Delta_0}{T \left(211\sigma^2+332 \contraction{\samplenum}\heterparam^2\right)L}},\sqrt[3]{\frac{\Delta_0}{250 T\contraction{\samplenum} L^2\left( \sigma^2 + \heterparam^2\right)}},\frac{1}{20L}\right\},
\end{equation}
we have 
\begin{align}\nonumber
    \frac{1}{\localstep} &= \max\left\{\sqrt{\frac{T \left(211\sigma^2+332 \contraction{\samplenum}\heterparam^2\right)L}{n\Delta_0}},\sqrt[3]{\frac{250 T\contraction{\samplenum} L^2\left( \sigma^2 + \heterparam^2\right)}{\Delta_0}},{20L}\right\}\\
    &\leq \sqrt{\frac{T \left(211\sigma^2+332 \contraction{\samplenum}\heterparam^2\right)L}{n\Delta_0}}+\sqrt[3]{\frac{250 T\contraction{\samplenum} L^2\left( \sigma^2 + \heterparam^2\right)}{\Delta_0}}+{20L}\label{eq:sqqq}
\end{align}
Plugging \eqref{eq:sqqq}, and~\eqref{sttt} in \eqref{eq:srrr} we obtain that
\begin{align*}
    \frac{1}{nT}\sum_{t = 0}^{T-1} \sum_{i \in [n]}\expect{ \norm{\nabla \loss \left( \model{i}{t}\right)}^2} &\leq 8\sqrt{\frac{L\Delta_0\left(211\sigma^2+ 332\contraction{\samplenum}\heterparam^2\right)}{nT}} 
     + 51 \sqrt[3]{\frac{\contraction{\samplenum} L^2\Delta_0^2\left( \sigma^2 + \heterparam^2\right)}{T^2}} + \frac{80L}{T} \Delta_0  \\
     &\in \mathcal{O} \left(\sqrt{\frac{L\Delta_0(\sigma^2+ \contraction{\samplenum}\heterparam^2)}{nT}} + \sqrt[3]{\frac{\contraction{\samplenum} L^2\Delta_0^2 \left( \sigma^2 + \heterparam^2\right)}{T^2}} + \frac{L\Delta_0}{T}\right).
\end{align*}
where we used the fact that $\frac{2000}{250^{(2/3)}}\leq 51$.
This concludes the proof.
\end{proof}

\section{Proof of the main lemmas}
\label{sec:lemma_proof}
{\bf Notation:}  Let $\mathcal{P}_t$ represent the history from step 0 to $t$. More precisely, we define
\[\mathcal{P}_t \coloneqq \left\{\model{i}{0}, \ldots, \, \model{i}{t}; ~ i = 1, \ldots, \, n \right\}.\] 
 Moreover, we use the notation $\condexpect{t}{\cdot} := \expect{\cdot ~ \vline ~ \mathcal{P}_t}$ to denote the conditional expectation given the history $\mathcal{P}_t$ and  $\expect{\cdot}$ to denote the total expectation over the randomness of the algorithm; therefore we have, $\expect{\cdot} := \condexpect{0}{ \cdots \condexpect{T}{\cdot}}$.

We first prove two simple useful lemmas.

\begin{lemma}
\label{lem:jensen_prime}
    for any set of $k$ vectors $x^{(1)}, \ldots, x^{(k)}$, we have
    \begin{align*}
        \norm{\sum_{i=1}^kx^{(i)}}^2\leq k\sum_{i=1}^k\norm{x^{(i)}}^2.
    \end{align*}
\end{lemma}
\begin{proof}
    By the triangle inequality, we have
     \begin{align*}
        \norm{\sum_{i=1}^kx^{(i)}}\leq \sum_{i=1}^k\norm{x^{(i)}}.
    \end{align*}
    The result then follows by noting that using Jensen's inequality, we have
    \begin{align*}
        \left(\frac{1}{k}\sum_{i=1}^k\norm{x^{(i)}}\right)^2\leq\frac{1}{k}\sum_{i=1}^k\norm{x^{(i)}}^2.
    \end{align*}
\end{proof}

\begin{lemma}
\label{lem:diam_equal}
    For any set $\{x^{(i)}\}_{i \in [n]}$ of $n$ vectors, we have $$ \frac{1}{n} \sum_{i\in [n]} \norm{ x^{(i)}- \bar{x}}^2 = \frac{1}{2} \cdot \frac{1}{n^2} \sum_{i,j \in [n]} \norm{ x^{(i)}- x^{(j)}}^2.$$
\end{lemma}
\begin{proof}
\begin{align*}
     \frac{1}{n^2} \sum_{i,j \in [n]} \norm{ x^{(i)}- x^{(j)}}^2 &= \frac{1}{n^2} \sum_{i,j \in [n]} \norm{ (x^{(i)} -\bar{x}) - (x^{(j)} - \bar{x})}^2\\ &= \frac{1}{n^2}\sum_{i,j \in [n]}\left[ \norm{x^{(i)} -\bar{x} }^2 + \norm{x^{(j)} -\bar{x} }^2 + 2 \iprod{x^{(i)} -\bar{x}}{x^{(j)} -\bar{x}}\right] \\
     &= \frac{2}{n} \sum_{i,j \in [n]} \norm{x^{(i)} -\bar{x} }^2  + \frac{2}{n^2}  \sum_{i \in [n]} \iprod{x^{(i)} - \bar{x}}{ \sum_{j \in [n]} (x^{(j)} -\bar{x})}.
\end{align*}

Noting that  $\sum_{j \in [n]} (x^{(j)} -\bar{x}) = 0$, yields the desired result.
\end{proof}
\subsection{Proof of Lemma~\ref{lem:contraction}}
\lemcontraction*
\begin{proof}
For simplicity of the notation, we make the dependence on $t$ implicit, and we denote by $\model{i}{} :=\model{i}{t+\nicefrac{1}{2}}$, the input vector and $\modelp{i}{} :=\model{i}{t+1}$ the output vector of node $i$ for the communication phase, and $\receivedsubset{i}{} := \receivedsubset{i}{t}$.
Therefore, we have 
\begin{align*}
    \modelp{i}{} = \frac{1}{\card{\receivedsubset{i}{} }+1} \left(\model{i}{} + \sum_{j \in \receivedsubset{i}{} } \model{j}{} \right).
\end{align*}
Moreover, in the proof of this lemma,  all the expectations $\expect{\cdot}$ are only with respect to the randomness of the communication phase, i.e., we compute the expectations assuming $\model{i}{}$'s are given for all $i \in [n]$, and $\modelp{i}{}$'s are computed by a random communication phase. Taking the total expectation of the final expression provided for each property then proves it as stated in the lemma.

First, we consider \textbf{\elo}:

Recall that  in \elo, the communication topology forms a random undirected $\samplenum$-regular graph, and for each node $i$, all other nodes have the same probability of being $i$'s neighbor.

\textbf{Property (a):}\\
We have
\begin{align*}
    \bar{\modelp{}{}} &= \frac{1}{n} 
    \sum_{i \in [n]} \modelp{i}{}\\
    &= \frac{1}{n} 
    \sum_{i \in [n]}\frac{1}{\card{\receivedsubset{i}{} }+1} \left(\model{i}{} + \sum_{j \in \receivedsubset{i}{} } \model{j}{} \right)
    \\
    &= \frac{1}{n(\samplenum+1)} 
    \sum_{i \in [n]}\left(\model{i}{} + \sum_{j \in \receivedsubset{i}{} } \model{j}{} \right)\\
    &= \frac{1}{n(\samplenum+1)} 
    \left(\sum_{i \in [n]}\model{i}{} + \sum_{i \in [n]}\sum_{j \in \receivedsubset{i}{} } \model{j}{} \right),
\end{align*}
where we used the fact that in \elo, each node receives exactly $\samplenum$ models. Now as each node also sends its model to $\samplenum$ other nodes, we have
\begin{align*}
    \bar{\modelp{}{}}
    &= \frac{1}{n(\samplenum+1)}  \left(\sum_{i \in [n]}\model{i}{} +
    \sum_{i \in [n]}\samplenum\model{i}{}\right) = \frac{1}{n} 
    \sum_{i \in [n]}\model{i}{} = \avgmodel{}{} .
\end{align*}

\textbf{Property (b):}\\
We have
\begin{align}\nonumber
    &\frac{1}{n} \sum_{i \in [n]} \expect{\norm{\modelp{i}{} -\bar{\model{}{}}}^2} = \frac{1}{n} \sum_{i \in [n]} \expect{\norm{ \frac{1}{\card{\receivedsubset{i}{} }+1} \left(\model{i}{} + \sum_{j \in \receivedsubset{i}{} } \model{j}{} \right)-\bar{\model{}{}}}^2}\\ \nonumber
    &= \frac{1}{n(s+1)^2} \sum_{i \in [n]} \expect{\norm{ (\model{i}{}-\bar{\model{}{}}) + \sum_{j \in \receivedsubset{i}{} } (\model{j}{}-\bar{\model{}{}}) }^2}\\ \nonumber
    &= \frac{1}{n(s+1)^2} \sum_{i \in [n]} {\norm{ \model{i}{}-\bar{\model{}{}}}^2} + \frac{1}{n(s+1)^2} \sum_{i \in [n]} \expect{\norm{ \sum_{j \in \receivedsubset{i}{} } (\model{j}{}-\bar{\model{}{}}) }^2}\\
    &+\frac{2}{n(s+1)^2} \sum_{i \in [n]} \expect{\iprod{ (\model{i}{}-\bar{\model{}{}}) }{ \sum_{j \in \receivedsubset{i}{} } (\model{j}{}-\bar{\model{}{}}) }}.\label{eq:intone}
\end{align}
Let us define $\mathcal{I}^{(i)}_j$ the indicator function 
showing whether node $j$ and node $i$ are neighbors, i.e., $\mathcal{I}^{(i)}_j=1$ if $i$ and $j$ are neighbors and $\mathcal{I}^{(i)}_j=0$ otherwise.  We then have
\begin{align}\nonumber
    &\sum_{i \in [n]}\expect{\norm{ \sum_{j \in \receivedsubset{i}{} } (\model{j}{}-\bar{\model{}{}}) }^2} =  \sum_{i \in [n]}\expect{\norm{ \sum_{j \in [n]\setminus\{i\} }\mathcal{I}^{(i)}_j (\model{j}{}-\bar{\model{}{}}) }^2} \\ \nonumber
    &=  \sum_{i \in [n]}\expect{\sum_{j \in [n]\setminus\{i\} }\mathcal{I}^{(i)}_j\norm{  \model{j}{}-\bar{\model{}{}} }^2} + \sum_{i\in[n]}\expect{\sum_{j \neq i}\sum_{k\neq i,k \neq j}\mathcal{I}^{(i)}_j\mathcal{I}^{(i)}_k\iprod{\model{j}{}-\bar{\model{}{}}}{\model{k}{}-\bar{\model{}{}}}}\\
    &=  \sum_{i \in [n]}\sum_{j \in [n]\setminus\{i\} }\expect{\mathcal{I}^{(i)}_j}\norm{  \model{j}{}-\bar{\model{}{}} }^2 + \sum_{i\in[n]}\sum_{j \neq i}\sum_{k\neq i,k \neq j}\expect{\mathcal{I}^{(i)}_j\mathcal{I}^{(i)}_k}\iprod{\model{j}{}-\bar{\model{}{}}}{\model{k}{}-\bar{\model{}{}}}\nonumber.
\end{align}
Now note that by symmetry (all nodes have the same probability of being a neighbor of node $i$).
\begin{align*}
    \expect{\mathcal{I}^{(i)}_j} = \frac{s}{n-1}.
\end{align*}
Similarly $\mathcal{I}^{(i)}_j \mathcal{I}^{(i)}_k$ is only $1$ when both $j$
and $k$ are neighbors of $i$, thus
\begin{align*}
    \expect{\mathcal{I}^{(i)}_j \mathcal{I}^{(i)}_k} = \frac{s(s-1)}{(n-1)(n-2)}.
\end{align*}
Therefore, 
\begin{align}\nonumber
    \sum_{i \in [n]}\expect{\norm{ \sum_{j \in \receivedsubset{i}{} } (\model{j}{}-\bar{\model{}{}}) }^2} 
    &= \samplenum \sum_{i \in [n]}\norm{  \model{j}{}-\bar{\model{}{}} }^2+ \sum_{i\in[n]}\sum_{j \neq i}\sum_{k\neq i,k \neq j} \frac{s(s-1)}{(n-1)(n-2)}\iprod{\model{j}{}-\bar{\model{}{}}}{\model{k}{}-\bar{\model{}{}}}\\
    &=\samplenum \sum_{i \in [n]}\norm{  \model{j}{}-\bar{\model{}{}} }^2+ \sum_{i\in[n]}\sum_{j \neq i} \frac{s(s-1)}{n-1}\iprod{\model{i}{}-\bar{\model{}{}}}{\model{j}{}-\bar{\model{}{}}}\label{eq:inttwo}.
\end{align}
Also,

\begin{align}\nonumber
    \sum_{i\in [n]}\expect{\iprod{ (\model{i}{}-\bar{\model{}{}}) }{ \sum_{j \in \receivedsubset{i}{} } (\model{j}{}-\bar{\model{}{}}) }} &= \sum_{i\in [n]}\iprod{ (\model{i}{}-\bar{\model{}{}}) }{ \sum_{j \neq i } \expect{\mathcal{I}^{(i)}_j} (\model{j}{}-\bar{\model{}{}}) }\\
     &= \frac{\samplenum}{n-1} \sum_{i\in [n]}\sum_{j \neq i }\iprod{ \model{i}{}-\bar{\model{}{}} }{   \model{j}{}-\bar{\model{}{}} }\label{eq:intthree}
\end{align}
Combining~\eqref{eq:intone},~\eqref{eq:inttwo}, and~\eqref{eq:intthree} we obtain that
\begin{align}\label{notefive}
    \frac{1}{n} \sum_{i \in [n]} \expect{\norm{\modelp{i}{} -\bar{\model{}{}}}^2}&=\frac{1+\samplenum}{n(s+1)^2} \sum_{i \in [n]} {\norm{ \model{i}{}-\bar{\model{}{}}}^2}\\
    &+ \frac{1}{n(s+1)^2}\left(\frac{2\samplenum}{n-1}+\frac{s(s-1)}{n-1}\right) \sum_{i\in [n]}\sum_{j \neq i }\iprod{ \model{i}{}-\bar{\model{}{}} }{   \model{j}{}-\bar{\model{}{}} }.\nonumber
\end{align}
Also note that
\begin{align}\label{eq:notefour}
     \sum_{i\in [n]}\sum_{j \neq i }\iprod{ \model{i}{}-\bar{\model{}{}} }{   \model{j}{}-\bar{\model{}{}}} = \sum_{i\in [n]}\iprod{ \model{i}{}-\bar{\model{}{}} }{ \sum_{j \neq i }  \model{j}{}-\bar{\model{}{}}} = -\sum_{i\in [n]}\norm{\model{i}{}-\bar{\model{}{}}}^2.
\end{align}
Combining~\eqref{notefive} and~\eqref{eq:notefour}, we obtain that
\begin{align}
    \frac{1}{n} \sum_{i \in [n]} \expect{\norm{\modelp{i}{} -\bar{\model{}{}}}^2} =\frac{1}{s+1}\left(1-\frac{s}{n-1}\right) \frac{1}{n}\sum_{i \in [n]}\norm{\model{i}{}-\bar{\model{}{}}}^2
    \label{eq:improne}
\end{align}
Now note that as $\bar{\modelp{}{}}$ is the minimizer of  $g(z) := \frac{1}{n} \sum_{i \in [n]} \expect{\norm{\modelp{i}{} -z}^2}$, we have 
\begin{align}
\label{impotwo}
    \frac{1}{n} \sum_{i \in [n]} \expect{\norm{\modelp{i}{} -\bar{\modelp{}{}}}^2} \leq  \frac{1}{n} \sum_{i \in [n]} \expect{\norm{\modelp{i}{} -\bar{\model{}{}}}^2}.
\end{align}
Combining this with~\eqref{eq:improne}, and using Lemma~\ref{lem:diam_equal} yields
\begin{align*}
    \frac{1}{n^2} \sum_{i,j \in [n]} \expect{\norm{\modelp{i}{} -\modelp{j}{}}^2} \leq  \frac{1}{s+1}\left(1-\frac{s}{n-1}\right)\frac{1}{n^2} \sum_{i,j \in [n]} \expect{\norm{\model{i}{} -\model{j}{}}^2},
\end{align*}
which is the desired result.

Now, we consider \textbf{\ello}:

Recall that  in \ello, each node $i$, locally samples uniformly at random  $\samplenum$ other nodes and sends its model to them.
Denote by $A^{(i)}:=\card{\receivedsubset{i}{}}$ the number of nodes that send their vectors to node $i$, and $\mathcal{I}^{(i)}_j$ the indicator showing whether node $j$ sends its vector to node $i$ or not ($\mathcal{I}^{(i)}_j= 1$  if $j$ sends its model to $i$, and $\mathcal{I}^{(i)}_j= 0$ otherwise).  We then have
\begin{align*}
    A^{(i)} = \sum_{j \in [n] \setminus \{i\}} \mathcal{I}^{(i)}_j,
\end{align*}
and 
\begin{align*}
    \expect{A^{(i)}} = \samplenum.
\end{align*}

\textbf{Property (a):} \\
For any node $i$, we have 
\begin{align*}
    \expect{\modelp{i}{}} &= \expect{\frac{1}{A^{(i)} + 1} \left(\model{i}{} + \sum_{j \in [n] \setminus \{i\}} \mathcal{I}^{(i)}_j\model{j}{} \right)}\\
    &= \expect{\expect{\frac{1}{A^{(i)} + 1} \left(\model{i}{} + \sum_{j \in [n] \setminus \{i\}} \mathcal{I}^{(i)}_j\model{j}{} \right)|A^{(i)}}}\\
    &= \expect{\frac{1}{A^{(i)} + 1} \left(\model{i}{} + \sum_{j \in [n] \setminus \{i\}}\expect {\mathcal{I}^{(i)}_j|A^{(i)}}\model{j}{} \right)}\\
    &\overset{(a)}{=} \expect{\frac{1}{A^{(i)} + 1} \left(\model{i}{} + \frac{A^{(i)}}{n-1} \sum_{j \in [n] \setminus \{i\}}\model{j}{} \right)}\\
    &= \expect{\frac{1}{A^{(i)} + 1} \left(\model{i}{} + \frac{A^{(i)}}{n-1} (n\bar{\model{}{}} - \model{i}{}) \right)},
\end{align*}
where (a) uses the fact that $\expect {\mathcal{I}^{(i)}_j|A^{(i)}} = \frac{A^{(i)}}{n-1}$ as all other $n-1$ nodes have the same probability of sending their model to $i$.
Denoting $p = \expect{\frac{A^{(i)}}{A^{(i)} + 1}}$, we obtain that
\begin{align*}
    \expect{\modelp{i}{}} = \frac{pn}{n-1} \bar{\model{}{}} + \left(1 -  \frac{pn}{n-1} \right) \model{i}{}.
\end{align*}
Averaging over all $i \in [n]$ yields the result.

\textbf{Property (b):} \\
We have
\begin{align*}
    &\frac{1}{n} \sum_{i \in [n]} \expect{\norm{\modelp{i}{} -\bar{\model{}{}}}^2} = \frac{1}{n} \sum_{i \in [n]} \expect{\norm{\frac{1}{A^{(i)} + 1} \left(\model{i}{} + \sum_{j \in [n] \setminus \{i\}} \mathcal{I}^{(i)}_j\model{j}{} \right) -\bar{\model{}{}}}^2}\\ 
    &= \frac{1}{n} \sum_{i \in [n]} \expect{\expect{\norm{\frac{1}{A^{(i)} + 1} \left(\model{i}{} + \sum_{j \in [n] \setminus \{i\}} \mathcal{I}^{(i)}_j\model{j}{} \right) -\bar{\model{}{}}}^2|A^{(i)}}}\\ 
    &= \frac{1}{n} \sum_{i \in [n]} \expect{\expect{\norm{\frac{1}{A^{(i)} + 1} \left((\model{i}{}-\bar{\model{}{}}) + \sum_{j \in [n] \setminus \{i\}} \mathcal{I}^{(i)}_j(\model{j}{}-\bar{\model{}{}}) \right)}^2|A^{(i)}}}\\ 
    &= \frac{1}{n} \sum_{i \in [n]} \expect{\frac{1}{(A^{(i)} + 1)^2} \expect{\norm{\model{i}{}-\bar{\model{}{}}}^2 + \sum_{j \neq i} \mathcal{I}^{(i)}_j\norm{\model{j}{}-\bar{\model{}{}}}^2}|A^{(i)}} \\
    &+ \frac{1}{n} \sum_{i \in [n]} \expect{\frac{1}{(A^{(i)} + 1)^2} \expect{2\sum_{j \neq i} \mathcal{I}^{(i)}_j \iprod{\model{i}{}-\bar{\model{}{}}}{\model{j}{}-\bar{\model{}{}}} + \sum_{j \neq i} \sum_{k \neq i,k \neq j}\mathcal{I}^{(i)}_j \mathcal{I}^{(i)}_k \iprod{\model{j}{}-\bar{\model{}{}}}{\model{k}{}-\bar{\model{}{}}}}|A^{(i)}}
\end{align*}
Taking the expectation inside, we obtain that
\begin{align*}
    &\frac{1}{n} \sum_{i \in [n]} \expect{\norm{\modelp{i}{} -\bar{\model{}{}}}^2} = \frac{1}{n} \sum_{i \in [n]} \expect{\frac{1}{(A^{(i)} + 1)^2} \left(\norm{\model{i}{}-\bar{\model{}{}}}^2 + \sum_{j \neq i} \expect{\mathcal{I}^{(i)}_j|A^{(i)}}\norm{\model{j}{}-\bar{\model{}{}}}^2\right)} \\
    &+ \frac{1}{n} \sum_{i \in [n]} \expect{\frac{1}{(A^{(i)} + 1)^2} \left(2\sum_{j \neq i} \expect{\mathcal{I}^{(i)}_j|A^{(i)}} \iprod{\model{i}{}-\bar{\model{}{}}}{\model{j}{}-\bar{\model{}{}}}\right)}
    \\
    &+ \frac{1}{n} \sum_{i \in [n]} \expect{\frac{1}{(A^{(i)} + 1)^2} \left( 
    \sum_{j \neq i} \sum_{k \neq i,k \neq j}\expect{\mathcal{I}^{(i)}_j \mathcal{I}^{(i)}_k|A^{(i)}} \iprod{\model{j}{}-\bar{\model{}{}}}{\model{k}{}-\bar{\model{}{}}}\right)}.
\end{align*}
Now note that $\expect{\mathcal{I}^{(i)}_j|A^{(i)}} $ is the probability that $j$ selects $i$, given that in total $A^{(i)}$ nodes select $i$, thus
\begin{align*}
    \expect{\mathcal{I}^{(i)}_j|A^{(i)}} = \frac{A^{(i)}}{n-1}.
\end{align*}
Similarly $\mathcal{I}^{(i)}_j \mathcal{I}^{(i)}_k$ is only $1$ when both $j$
and $k$ select $i$, thus
\begin{align*}
    \expect{\mathcal{I}^{(i)}_j \mathcal{I}^{(i)}_k|A^{(i)}} = \frac{A^{(i)}(A^{(i)}-1)}{(n-1)(n-2)}
\end{align*}
Also, note that 
\begin{align*}
    \sum_{j \neq i}\iprod{\model{i}{}-\bar{\model{}{}}}{\model{j}{}-\bar{\model{}{}}} = \iprod{\model{i}{}-\bar{\model{}{}}}{\sum_{j \neq i}(\model{j}{}-\bar{\model{}{}})} = - \norm{\model{i}{}-\bar{\model{}{}}}^2,
\end{align*}
and
\begin{align*}
    \sum_{j \neq i} \sum_{k \neq i,k \neq j}\iprod{\model{j}{}-\bar{\model{}{}}}{\model{k}{}-\bar{\model{}{}}}
    &= \sum_{j \neq i} \iprod{\model{j}{}-\bar{\model{}{}}}{\sum_{k \neq i,k \neq j}(\model{k}{}-\bar{\model{}{}})}\\
    &=- \sum_{j \neq i} \iprod{\model{j}{}-\bar{\model{}{}}}{(\model{i}{}-\bar{\model{}{}}) + (\model{j}{}-\bar{\model{}{}})}\\
    &= \norm{\model{i}{}-\bar{\model{}{}}}^2 - \sum_{j \neq i} \norm{\model{j}{}-\bar{\model{}{}}}^2.
\end{align*}
Combining all yields
\begin{align*}
    &\frac{1}{n} \sum_{i \in [n]} \expect{\norm{\modelp{i}{} -\bar{\model{}{}}}^2} = \frac{1}{n} \sum_{i \in [n]} \expect{\frac{1}{(A^{(i)} + 1)^2} \left(\norm{\model{i}{}-\bar{\model{}{}}}^2 + \frac{A^{(i)}}{n-1}\sum_{j \neq i} \norm{\model{j}{}-\bar{\model{}{}}}^2\right)} \\
    &+ \frac{1}{n} \sum_{i \in [n]} \expect{\frac{1}{(A^{(i)} + 1)^2} \left(-\frac{2A^{(i)}}{n-1}\norm{\model{i}{}-\bar{\model{}{}}}^2\right)}
    \\
    &+ \frac{1}{n} \sum_{i \in [n]} \expect{\frac{1}{(A^{(i)} + 1)^2} \cdot \frac{A^{(i)}(A^{(i)}-1)}{(n-1)(n-2)} \left( 
    \norm{\model{i}{}-\bar{\model{}{}}}^2 - \sum_{j \neq i} \norm{\model{j}{}-\bar{\model{}{}}}^2\right)}\\
    &=  \frac{1}{n} \sum_{i \in [n]} \norm{\model{i}{}-\bar{\model{}{}}}^2 \expect{\frac{1}{(A^{(i)} + 1)^2} \left(1-\frac{2A^{(i)}}{n-1} + \frac{A^{(i)}(A^{(i)}-1)}{(n-1)(n-2)}\right)}\\
    &+ \frac{1}{n} \sum_{i \in [n]} \left( \expect{\frac{1}{(A^{(i)} + 1)^2} \left(\frac{A^{(i)}}{n-1} - \frac{A^{(i)}(A^{(i)}-1)}{(n-1)(n-2)}\right)} \sum_{j \neq i} \norm{\model{j}{}-\bar{\model{}{}}}^2 \right)
\end{align*}
Now note that by symmetry, the distribution of $A^{(i)}$ is the same as $A^{(j)}$ for any $i,j \in [n]$. Therefore,
\begin{align*}
    &\frac{1}{n} \sum_{i \in [n]} \expect{\norm{\modelp{i}{} -\bar{\model{}{}}}^2}\\ 
    &=  \frac{1}{n} \sum_{i \in [n]} \norm{\model{i}{}-\bar{\model{}{}}}^2 \expect{\frac{1}{(A^{(1)} + 1)^2} \left(1-\frac{2A^{(1)}}{n-1} + \frac{A^{(1)}(A^{(1)}-1)}{(n-1)(n-2)}+{A^{(1)}} - \frac{A^{(1)}(A^{(1)}-1)}{n-2}\right)}
\end{align*}
Now note that 
\begin{align*}
    1-\frac{2A^{(1)}}{n-1} + \frac{A^{(1)}(A^{(1)}-1)}{(n-1)(n-2)}+{A^{(1)}} - \frac{A^{(1)}(A^{(1)}-1)}{n-2} = 1 + {A^{(1)}} - \frac{{A^{(1)}}^2+A^{(1)}}{n-1} = (1 + {A^{(1)}})(1-\frac{A^{(1)}}{n-1}),
\end{align*}
thus
\begin{align*}
    &\frac{1}{n} \sum_{i \in [n]} \expect{\norm{\modelp{i}{} -\bar{\model{}{}}}^2} = \frac{1}{n} \sum_{i \in [n]} \norm{\model{i}{}-\bar{\model{}{}}}^2 \left(\expect{\frac{1}{A^{(1)} + 1}} - \frac{1}{n-1} \cdot \expect{\frac{A^{(1)}}{A^{(1)} + 1}}\right)
\end{align*}
Now note that as each node $j \neq 1$, independently and uniformly selects a set of $\samplenum$ nodes,  $A^{(1)}$ has a binomial distribution with parameters $n-1$, and $\frac{\samplenum}{n-1}$. Therefore, for $\samplenum>0$, we have
\begin{align*}
    \expect{\frac{1}{A^{(1)} + 1}} &= \sum_{k = 0}^{n-1} \frac{1}{k+1} {n-1 \choose k} \left( \frac{\samplenum}{n-1}\right)^k \left(1-\frac{\samplenum}{n-1} \right)^{n-1-k}\\ 
    &= \frac{n-1}{\samplenum n}  \sum_{k = 0}^{n-1} {n \choose k+1} \left( \frac{\samplenum}{n-1}\right)^{k+1} \left(1-\frac{\samplenum}{n-1} \right)^{n-1-k}\\ 
    &=  \frac{n-1}{\samplenum n} \left(1- \left(1 - \frac{\samplenum}{n-1} \right)^{n} \right).
\end{align*}
Also noting that $$\expect{\frac{A^{(1)}}{A^{(1)} + 1}} = 1- \expect{\frac{1}{A^{(1)} + 1}},$$
we obtain that

\begin{align}
\label{eq:impothree}
    &\frac{1}{n} \sum_{i \in [n]} \expect{\norm{\modelp{i}{} -\bar{\model{}{}}}^2} =  \left[\frac{1}{\samplenum}\left(1-\left(1-\frac{\samplenum}{n-1}\right)^n\right)-\frac{1}{n-1}\right]\frac{1}{n} \sum_{i \in [n]} \norm{\model{i}{}-\bar{\model{}{}}}^2.
\end{align}
Noting that 
$
    \frac{1}{n} \sum_{i \in [n]} \expect{\norm{\modelp{i}{} -\bar{\modelp{}{}}}^2} \leq  \frac{1}{n} \sum_{i \in [n]} \expect{\norm{\modelp{i}{} -\bar{\model{}{}}}^2}
$, from~\eqref{impotwo}, and using Lemma~\ref{lem:diam_equal}, we obtain that
\begin{align*}
    \frac{1}{n^2} \sum_{i,j \in [n]} \expect{\norm{\modelp{i}{} -\modelp{j}{}}^2} \leq  \left[\frac{1}{\samplenum}\left(1-\left(1-\frac{\samplenum}{n-1}\right)^n\right)-\frac{1}{n-1}\right]\frac{1}{n^2} \sum_{i,j \in [n]} \expect{\norm{\model{i}{} -\model{j}{}}^2},
\end{align*}
which is the desired result.

\textbf{Property (c):}\\
Note that 
\begin{align}\nonumber
    \expect{\norm{\bar{\modelp{}{}}-\bar{\model{}{}}}^2}  &= \expect{\norm{\frac{1}{n}\sum_{i \in [n]} \modelp{i}{}-\bar{\model{}{}}}^2}\\
    & = \frac{1}{n^2} \sum_{i \in [n]} \expect{\norm{ \modelp{i}{}-\bar{\model{}{}}}^2} + \frac{1}{n^2} \sum_{i \neq j} \expect{\iprod{\modelp{i}{}-\bar{\model{}{}}}{\modelp{j}{}-\bar{\model{}{}}}}. \label{eq:variance_y_bar}
\end{align}
Now recall that 
\begin{align*}
    \modelp{i}{}-\bar{\model{}{}} = \frac{1}{A^{(i)} + 1} \left((\model{i}{}-\bar{\model{}{}}) + \sum_{k \in [n] \setminus \{i\}} \mathcal{I}^{(i)}_k(\model{k}{}-\bar{\model{}{}}) \right).
\end{align*}
This implies that 
\begin{align}\nonumber
    &\frac{1}{n^2} \sum_{i \neq j} \expect{\iprod{\modelp{i}{}-\bar{\model{}{}}}{\modelp{j}{}-\bar{\model{}{}}}} =  \frac{1}{n^2} \sum_{i \in [n]} \sum_{j \neq i} \expect{\frac{1}{(A^{(i)} + 1)(A^{(j)} + 1)}}\iprod{\model{i}{}-\bar{\model{}{}}}{\model{j}{}-\bar{\model{}{}}}\\ \nonumber
    &+ \frac{2}{n^2}  \sum_{i \in [n]} \sum_{j \neq i} \sum_{k \neq i, k\neq j}  \expect{\frac{\mathcal{I}^{(i)}_k}{(A^{(i)} + 1)(A^{(j)} + 1)}}\iprod{\model{k}{}-\bar{\model{}{}}}{\model{j}{}-\bar{\model{}{}}}\\
    &+ \frac{2}{n^2}  \sum_{i \in [n]} \sum_{j \neq i} \sum_{k \neq i, k\neq j}   \sum_{l \neq i, l\neq j,l\neq k} \expect{\frac{\mathcal{I}^{(i)}_k\mathcal{I}^{(j)}_l}{(A^{(i)} + 1)(A^{(j)} + 1)}}\iprod{\model{k}{}-\bar{\model{}{}}}{\model{l}{}-\bar{\model{}{}}}\label{eq:big_messy_eq}
\end{align}
Now note that by symmetry, for any $i,j \in [n]$, we have
\begin{align*}
    \expect{\frac{1}{(A^{(i)} + 1)(A^{(j)} + 1)}} = \expect{\frac{1}{(A^{(1)} + 1)(A^{(2)} + 1)}}.
\end{align*}
Similarly,
\begin{align*}
    \expect{\frac{\mathcal{I}^{(i)}_k}{(A^{(i)} + 1)(A^{(j)} + 1)}} = \expect{\frac{\mathcal{I}^{(1)}_3}{(A^{(1)} + 1)(A^{(2)} + 1)}},
\end{align*}
and 
\begin{align*}
    \expect{\frac{\mathcal{I}^{(i)}_k\mathcal{I}^{(j)}_l}{(A^{(i)} + 1)(A^{(j)} + 1)}} = \expect{\frac{\mathcal{I}^{(1)}_3\mathcal{I}^{(2)}_4}{(A^{(1)} + 1)(A^{(2)} + 1)}}.
\end{align*}
This implies that all three terms in~\eqref{eq:big_messy_eq} can be written as 
\begin{align*}
     c \sum_{i \in [n]}\sum_{j \neq i} \iprod{\model{i}{}-\bar{\model{}{}}}{\model{j}{}-\bar{\model{}{}}}, 
\end{align*}
where $c$ is a positive constant.
We also have
\begin{align*}
    \sum_{i \in [n]}\sum_{j \neq i} \iprod{\model{i}{}-\bar{\model{}{}}}{\model{j}{}-\bar{\model{}{}}} &= \sum_{i \in [n]} \iprod{\model{i}{}-\bar{\model{}{}}}{\sum_{j \neq i}(\model{j}{}-\bar{\model{}{}})}\\
    &= - \sum_{i \in [n]} \norm{\model{i}{}-\bar{\model{}{}}}^2.
\end{align*}
Therefore all the terms in~\eqref{eq:big_messy_eq} are non-positive. Combining this with~\eqref{eq:variance_y_bar}, we obtain that
\begin{align*}
    \expect{\norm{\bar{\modelp{}{}}-\bar{\model{}{}}}^2}  &\leq \frac{1}{n^2} \sum_{i \in [n]} \expect{\norm{ \modelp{i}{}-\bar{\model{}{}}}^2} \\
    &\leq \frac{\contraction{\samplenum}}{n} \cdot  \frac{1}{n} \sum_{i \in [n]} \norm{\model{i}{} -\bar{\model{}{}}}^2,
\end{align*}
where the second inequality uses~\eqref{eq:impothree}. Combining this with Lemma~\ref{lem:diam_equal} then concludes the proof.
\end{proof}
\begin{remark}
\label{rem:triv}
Note that $\contraction{\samplenum}$ computed in Lemma~\ref{lem:contraction} is decreasing in $\samplenum$ and increasing in $n$, therefore, for any $\samplenum \geq 1$, and $n \geq 2$, we have
\begin{align*}
    \contraction{\samplenum}\Big|_{n < \infty} &\leq \lim_{n \rightarrow \infty} \contraction{1}\\
    &=
\lim_{n\to\infty}{\left(1-\left(1-\frac{1}{n-1}\right)^{n}-\frac{1}{n-1}\right)}\\
&= 1 - \frac{1}{e}, 
\end{align*}
where $e$ is Euler's Number and we used the fact that $\lim_{n\to\infty}{\left(1-\frac{1}{n}\right)^{n}}
= \frac{1}{e}$. Similarly,
for any $\samplenum \geq 1$, and $n \geq 2$, we have
\begin{align*}
    \contractionp{\samplenum}\Big|_{n < \infty} &\leq \lim_{n \rightarrow \infty} \contractionp{1}= \frac{1}{2} < 1 - \frac{1}{e}.
\end{align*}
\end{remark}

\subsection{Proof of Lemma~\ref{lem:SGD:drift}}
\lemSGDdrift*

\begin{proof}

First note that for any $i \in [n]$, we have 
\begin{align*}
    \gradient{i}{t}-\gradient{j}{t}  &= \gradient{i}{t} - \localgrad{i}\left( \model{i}{t}\right) + \localgrad{i}\left( \model{i}{t}\right) - \localgrad{i}\left( \avgmodel{t}\right) + \localgrad{i}\left( \avgmodel{t}\right) \\ &- 
    \localgrad{j}\left( \avgmodel{t}\right) +
    \localgrad{j}\left( \avgmodel{t}\right) -
    \localgrad{j}\left( \model{j}{t}\right) +
    \localgrad{j}\left( \model{j}{t}\right) -\gradient{j}{t}
\end{align*}
Thus, using Lemma~\ref{lem:jensen_prime}, we have
\begin{align*}
    {\norm{ \gradient{i}{t}-\gradient{j}{t}}^2} &\leq 5{\norm{\gradient{i}{t} - \localgrad{i}\left( \model{i}{t}\right)}^2}+5{\norm{\localgrad{i}\left( \model{i}{t}\right) - \localgrad{i}\left( \avgmodel{t}\right)}^2}\\ 
    &+5{\norm{\localgrad{j}\left( \model{j}{t}\right) - \localgrad{j}\left( \avgmodel{t}\right)}^2}
    +5{\norm{\gradient{j}{t} - \localgrad{j}\left( \model{j}{t}\right)}^2}\\
    &+5{\norm{\localgrad{i}\left( \avgmodel{t}\right) - \localgrad{j}\left( \avgmodel{t}\right)}^2}.
\end{align*}
Taking the conditional expectation, we have
\begin{align}\nonumber
    \condexpect{t}{\norm{ \gradient{i}{t}-\gradient{j}{t}}^2} &\leq 5\condexpect{t}{\norm{\gradient{i}{t} - \localgrad{i}\left( \model{i}{t}\right)}^2}+5\condexpect{t}{\norm{\localgrad{i}\left( \model{i}{t}\right) - \localgrad{i}\left( \avgmodel{t}\right)}^2}\\ \nonumber
    &+5\condexpect{t}{\norm{\localgrad{j}\left( \model{j}{t}\right) - \localgrad{j}\left( \avgmodel{t}\right)}^2}
    +5\condexpect{t}{\norm{\gradient{j}{t} - \localgrad{j}\left( \model{j}{t}\right)}^2}\\
    &+5\condexpect{t}{\norm{\localgrad{i}\left( \avgmodel{t}\right) - \localgrad{j}\left( \avgmodel{t}\right)}^2}\label{eq:temp1}
\end{align}
Now by Assumption~\ref{ass:stochastic_noise}, we have
\begin{align}
\label{eq:temp2}
    \condexpect{t}{\norm{\gradient{i}{t} - \localgrad{i}\left( \model{i}{t}\right)}^2} \leq \sigma^2.
\end{align}
By Assumption~\ref{ass:smoothness}, we have
\begin{align}
\label{eq:temp3}
    \condexpect{t}{\norm{\localgrad{i}\left( \model{i}{t}\right) - \localgrad{i}\left( \avgmodel{t}\right)}^2} \leq L^2 \condexpect{t}{\norm{\model{i}{t} - \avgmodel{t}}^2}.
\end{align}
Thus, by Assumption~\ref{ass:bounded_heterogeneity}, and Lemma~\ref{lem:diam_equal}, we obtain that
\begin{align}
    \label{eq:temp4}
    \frac{1}{n^2} \sum_{i,j\in[n]} \condexpect{t}{\norm{\localgrad{i}\left( \avgmodel{t}\right) - \localgrad{j}\left( \avgmodel{t}\right)}^2} \leq 2 \heterparam^2.
\end{align}
Combining~\eqref{eq:temp1},~\eqref{eq:temp2},~\eqref{eq:temp3}, and~\eqref{eq:temp4}, and taking total expectation from both sides, we obtain that
\begin{align}
\label{eq:q_boundi}
    &\frac{1}{n^2}\sum_{i,j \in [n]}\expect{\norm{ \gradient{i}{t}-\gradient{j}{t}}^2} \leq \frac{10L^2 }{n} \sum_{i \in [n]}\expect{\norm{\model{i}{t} - \avgmodel{t}}^2} + 10\sigma^2 + 10\heterparam^2.
\end{align}
Now  Lemma~\ref{lem:diam_equal} yields
\begin{align}\label{eq:g_boud}
     &\frac{1}{n^2}\sum_{i,j \in [n]}\expect{\norm{ \gradient{i}{t}-\gradient{j}{t}}^2} \leq \frac{5L^2 }{n^2} \sum_{i,j \in [n]}\expect{\norm{\model{i}{t} - \model{j}{t}}^2} + 10\sigma^2 + 10\heterparam^2.
\end{align}

We now analyze  $\frac{1}{n^2} \sum_{i,j \in [n]} \expect{\norm{\model{i}{t} - \model{j}{t}}^2}$. From Algorithm \ref{algo}, recall that for all $i \in [n]$, we have $\model{i}{t+\nicefrac{1}{2}} = \model{i}{t} - \localstep \gradient{i}{t}$. We obtain for all $i, \, j \in [n]$, that
\begin{align}\label{eq:distanc_models_}
    \expect{\norm{\model{i}{t+\nicefrac{1}{2}} - \model{j}{t+\nicefrac{1}{2}} }^2} & \leq \expect{\norm{\model{i}{t} - \model{j}{t} - \localstep \left( \gradient{i}{t} - \gradient{j}{t} \right)}^2} \\
    & \leq (1 + c) \expect{\norm{\model{i}{t} - \model{j}{t}}^2} + \left(1 + \frac{1}{c} \right) \localstep^2 \expect{\norm{\gradient{i}{t} - \gradient{j}{t}}^2},\nonumber
\end{align}
where we used the fact that $(x + y)^2 \leq (1 + c) x^2 + (1 + \nicefrac{1}{c}) y^2$ for any $c > 0$.
Now recall that $\contractionb{\samplenum} = \contractionp{\samplenum}$ for \elo, and  $\contractionb{\samplenum} = \contraction{\samplenum}$ for~\ello. Therefore,  by Property $(b)$ of Lemma~\ref{lem:contraction}, for both variant of the algorithm, we have
\begin{align*}
    \frac{1}{n} \sum_{i \in [n]} \expect{\norm{\model{i}{t + 1} -\avgmodel{t+1}}^2}  \leq \frac{\contractionb{\samplenum}}{n} \sum_{i \in [n]} \expect{\norm{\model{i}{t + \nicefrac{1}{2}} -\avgmodel{t+\nicefrac{1}{2}}}^2} 
\end{align*}
Then, applying Lemma~\ref{lem:diam_equal}, we have
\begin{align*}
    \frac{1}{n^2}\sum_{i,j\in[n]}\expect{\norm{\model{i}{t + 1} -\model{j}{t + 1}}^2} \leq  \frac{\contractionb{\samplenum}}{n^2}\sum_{i,j\in[n]}\expect{\norm{\model{i}{t + \nicefrac{1}{2}} -\model{j}{t + \nicefrac{1}{2}}}^2}.
\end{align*}
Combining this with~\eqref{eq:distanc_models_}, we obtain that 
\begin{align*}
    \frac{1}{n^2}\sum_{i,j\in[n]}\expect{\norm{\model{i}{t+1} - \model{j}{t+1} }^2} 
    & \leq (1 + c) \contractionb{\samplenum} \frac{1}{n^2}\sum_{i,j\in[n]} \expect{\norm{\model{i}{t} - \model{j}{t}}^2}\\ &+ \left(1 + \frac{1}{c} \right) \contractionb{\samplenum} \localstep^2 \frac{1}{n^2}\sum_{i,j\in[n]}\expect{\norm{\gradient{i}{t} - \gradient{j}{t}}^2}.
\end{align*}

  For $c = \frac{1-\contractionb{\samplenum}}{4\contractionb{\samplenum}}$, we obtain that
\begin{align*}
    \frac{1}{n^2}\sum_{i,j\in[n]}\expect{\norm{\model{i}{t+1} - \model{j}{t+1} }^2} 
    & \leq \frac{1+3\contractionb{\samplenum}}{4} \frac{1}{n^2}\sum_{i,j\in[n]} \expect{\norm{\model{i}{t} - \model{j}{t}}^2}\\ &+ \frac{1+3\contractionb{\samplenum}}{1-\contractionb{\samplenum}} \contractionb{\samplenum} \localstep^2 \frac{1}{n^2}\sum_{i,j\in[n]}\expect{\norm{\gradient{i}{t} - \gradient{j}{t}}^2}.
\end{align*}
Combining this with~\eqref{eq:g_boud}, we obtain that 
\begin{align*}
    \frac{1}{n^2}\sum_{i,j\in[n]}\expect{\norm{\model{i}{t+1} - \model{j}{t+1} }^2} 
    & \leq \frac{1+3\contractionb{\samplenum}}{4} \frac{1}{n^2}\sum_{i,j\in[n]} \expect{\norm{\model{i}{t} - \model{j}{t}}^2}\\ &+ \frac{1+3\contractionb{\samplenum}}{1-\contractionb{\samplenum}} \contractionb{\samplenum} \localstep^2 \left( \frac{5L^2 }{n^2} \sum_{i,j \in [n]}\expect{\norm{\model{i}{t} - \model{j}{t}}^2} + 10\sigma^2 + 10\heterparam^2\right)\\
    &= \left(\frac{1+3\contractionb{\samplenum}}{4} + 5\frac{1+3\contractionb{\samplenum}}{1-\contractionb{\samplenum}}  \contractionb{\samplenum} \localstep^2L^2\right)\frac{1}{n^2}\sum_{i,j\in[n]} \expect{\norm{\model{i}{t} - \model{j}{t}}^2}\\
    &+\frac{1+3\contractionb{\samplenum}}{1-\contractionb{\samplenum}} \contractionb{\samplenum}\localstep^2 \left( 10\sigma^2 + 10\heterparam^2\right).
\end{align*}

Now note that for by Remark~\ref{rem:triv}, for both variants of the algorithm, we have $\contractionb{\samplenum}\leq 1 - \frac{1}{e}$, which implies that 
 $$\localstep^2\leq \frac{1}{(20L)^2}  \leq \frac{(1-\contractionb{\samplenum})^2}{20\contractionb{\samplenum}(1+3\contractionb{\samplenum})L^2}.$$ Therefore,
\begin{align*}
    \frac{1}{n^2}\sum_{i,j\in[n]}\expect{\norm{\model{i}{t+1} - \model{j}{t+1} }^2} 
    &  \leq \frac{1+\contractionb{\samplenum}}{2}\frac{1}{n^2}\sum_{i,j\in[n]} \expect{\norm{\model{i}{t} - \model{j}{t}}^2}+\frac{1+3\contractionb{\samplenum}}{1-\contractionb{\samplenum}} \contractionb{\samplenum}\localstep^2 \left( 10\sigma^2 + 10\heterparam^2\right).
\end{align*}
Unrolling the recursion, we obtain that

\begin{align*}
    \frac{1}{n^2}\sum_{i,j\in[n]}\expect{\norm{\model{i}{t} - \model{j}{t} }^2} 
    & \leq 20\frac{1+3\contractionb{\samplenum}}{(1-\contractionb{\samplenum})^2} \contractionb{\samplenum}\localstep^2 \left( \sigma^2 + \heterparam^2\right).
\end{align*}
Combining this with~\eqref{eq:q_boundi}, we obtain that 
\begin{align*}
    \frac{1}{n^2}\sum_{i,j \in [n]}\expect{\norm{ \gradient{i}{t}-\gradient{j}{t}}^2}\leq 15\left(   \sigma^2 +  \heterparam^2 \right).
\end{align*}
This is the desired result.
\end{proof}

\subsection{Proof of Lemma~\ref{lem:growth_SGD}}
\lemgrowthSGD*

\begin{proof}
\newcommand{\bias}[1]{R_{#1}}
Consider an arbitrary $t \in [T]$.
Note that, we do not necessarily have 
$\avgmodel{t + 1} = \avgmodel{t + \nicefrac{1}{2}}$. But by Lemma~\ref{lem:contraction}, we have $\expect{\avgmodel{t + 1}} = \expect{\avgmodel{t + \nicefrac{1}{2}}}$. Now 
by the smoothness of the loss function (Assumption~\ref{ass:smoothness}), we have 
\begin{align*}
    \loss(\avgmodel{t+1})- \loss(\avgmodel{t}) &\leq \iprod{\avgmodel{t+1} - \avgmodel{t}}{\nabla \loss\left( \avgmodel{t} \right)} + \frac{L}{2} \norm{\avgmodel{t+1} - \avgmodel{t}}^2 \nonumber \\
    & = \iprod{\avgmodel{t+1} -\avgmodel{t+\nicefrac{1}{2}} + \avgmodel{t+\nicefrac{1}{2}} - \avgmodel{t}}{ \nabla \loss \left( \avgmodel{t} \right)} + \frac{L}{2} \norm{\avgmodel{t+1} -\avgmodel{t+\nicefrac{1}{2}} + \avgmodel{t+\nicefrac{1}{2}} - \avgmodel{t}}^2.
\end{align*}
Now denoting $\AvgGrad{t} = \frac{1}{n} \sum_{i \in [n]} \localgrad{i}\left(\model{i}{t}\right)$, $\avggrad{t} = \frac{1}{n} \sum_{i \in [n]} \gradient{i}{t}$ and taking the conditional expectation, we have
\begin{align}\nonumber
    \condexpect{t}{\loss(\avgmodel{t+1})}- \loss(\avgmodel{t}) &\leq  \iprod{\condexpect{t}{\avgmodel{t+\nicefrac{1}{2}} - \avgmodel{t}}} { \nabla \loss \left( \avgmodel{t} \right)}  + \frac{L}{2} \condexpect{t}{\norm{\avgmodel{t+1} -\avgmodel{t+\nicefrac{1}{2}} + \avgmodel{t+\nicefrac{1}{2}} - \avgmodel{t}}^2} \\ \nonumber
    &\overset{(a)}{\leq} - \localstep \iprod{ \AvgGrad{t}} { \nabla \loss \left( \avgmodel{t} \right)}  + {L\localstep^2} \condexpect{t}{\norm{\avggrad{t}}^2} + {L} \condexpect{t}{\norm{\avgmodel{t+1} -\avgmodel{t+\nicefrac{1}{2}}}^2} \\
    &\overset{(b)}{\leq} - \localstep \iprod{ \AvgGrad{t}} { \nabla \loss \left( \avgmodel{t} \right)}  + {L\localstep^2} \condexpect{t}{\norm{ \AvgGrad{t}}^2} +  {L\localstep^2 \frac{\sigma^2}{n}} + {{L} \condexpect{t}{\norm{\avgmodel{t+1} -\avgmodel{t+\nicefrac{1}{2}}}^2}} \label{eq:groth_first}.
\end{align}
where  (a) uses Young's inequality and (b) is based on the facts that by Assumption~\ref{ass:stochastic_noise}, we have  $\condexpect{t}{\avggrad{t}} = \AvgGrad{t}$, and $\condexpect{t}{\norm{\avggrad{t}- \AvgGrad{t}}^2}  \leq \frac{\sigma^2}{n}$.
Now note that 
using the fact that $\localstep \leq \frac{1}{2L}$, we have
\begin{align}\nonumber
    - \localstep \iprod{ \AvgGrad{t}} { \nabla \loss \left( \avgmodel{t} \right)}  + {L\localstep^2} {\norm{ \AvgGrad{t}}^2}
    &\leq \frac{\localstep}{2} \left(-2  \iprod{ \AvgGrad{t}} { \nabla \loss \left( \avgmodel{t} \right)}  +{\norm{ \AvgGrad{t}}^2} \right)\\
    &= \frac{\localstep}{2} \left(  -\norm{\nabla \loss \left( \avgmodel{t} \right)}^2 + \norm{\nabla \loss \left( \avgmodel{t} \right) -  \AvgGrad{t}}^2 \right).
    \label{eq:groth_middle}
\end{align}
Combining  \eqref{eq:groth_first}, and \eqref{eq:groth_middle}, we obtain that
\begin{align*}
     \condexpect{t}{\loss(\avgmodel{t+1})}- \loss(\avgmodel{t}) &\leq -\frac{\localstep}{2} \norm{\nabla \loss \left( \avgmodel{t} \right)}^2 + \frac{\localstep}{2} \norm{\nabla \loss \left( \avgmodel{t} \right) -  \AvgGrad{t}}^2 + {L\localstep^2 \frac{\sigma^2}{n}} + {L\condexpect{t}{\norm{\avgmodel{t+1} -\avgmodel{t+\nicefrac{1}{2}}}^2}}.
\end{align*}
Taking total expectation, we obtain that
\begin{align}\nonumber
     \expect{\loss(\avgmodel{t+1})- \loss(\avgmodel{t})} &\leq -\frac{\localstep}{2} \expect{\norm{\nabla \loss \left( \avgmodel{t} \right)}^2} + \frac{\localstep}{2} \expect{\norm{\nabla \loss \left( \avgmodel{t} \right) -  \AvgGrad{t}}^2}\\ 
     & + {L\localstep^2 \frac{\sigma^2}{n}} + {L\expect{\norm{\avgmodel{t+1} -\avgmodel{t+\nicefrac{1}{2}}}^2}}.\label{eq:desenet_bound}
\end{align}

Now, note that 
\begin{align}\nonumber
    \expect{\norm{\AvgGrad{t}-\nabla \loss \left( \avgmodel{t} \right)}^2} &= \expect{\norm{\frac{1}{n} \sum_{i \in [n]} \localgrad{i}(\model{i}{t}) - \frac{1}{n} \sum_{i \in [n]} \localgrad{i}\left(\avgmodel{t}\right)}^2}\\ \nonumber
    &= \expect{\norm{\frac{1}{n} \sum_{i \in [n]} \left(\localgrad{i}(\model{i}{t}) - \localgrad{i}\left(\avgmodel{t}\right)\right)}^2} \\ \nonumber
    &\leq \frac{1}{n} \sum_{i \in [n]} \expect{\norm{\localgrad{i}(\model{i}{t}) - \localgrad{i}\left(\avgmodel{t}\right)}^2}\\
    &\overset{(a)}{\leq} \frac{L^2}{n} \sum_{i \in [n]} \expect{\norm{\model{i}{t} - \avgmodel{t}}^2}.\nonumber\\
    &\overset{(b)}{\leq} \frac{L^2}{2n^2} \sum_{i,j \in [n]} \expect{\norm{\model{i}{t} - \model{j}{t}}^2},\nonumber
\end{align}
where (a) uses Assumption~\ref{ass:smoothness}, in (b), we used Lemma~\ref{lem:diam_equal}.
Combining this with~\eqref{eq:desenet_bound}, we obtain that
\begin{align}\nonumber
     \expect{\loss(\avgmodel{t+1})- \loss(\avgmodel{t})} &\leq -\frac{\localstep}{2} \expect{\norm{\nabla \loss \left( \avgmodel{t} \right)}^2} + \frac{\localstep}{2} \frac{L^2}{2n^2} \sum_{i,j \in [n]} \expect{\norm{\model{i}{t} - \model{j}{t}}^2}\\ 
     & + {L\localstep^2 \frac{\sigma^2}{n}} + {L\expect{\norm{\avgmodel{t+1} -\avgmodel{t+\nicefrac{1}{2}}}^2}}.\nonumber
\end{align}
Rearranging the terms we have
\begin{align}\nonumber
     \expect{\norm{\nabla \loss \left( \avgmodel{t} \right)}^2}&\leq \frac{2}{\localstep}\expect{\loss(\avgmodel{t})- \loss(\avgmodel{t+1})} +\frac{L^2}{2n^2} \sum_{i,j \in [n]} \expect{\norm{\model{i}{t} - \model{j}{t}}^2}\\ 
     & + {2L\localstep \frac{\sigma^2}{n}} + \frac{2L}{\localstep}{\expect{\norm{\avgmodel{t+1} -\avgmodel{t+\nicefrac{1}{2}}}^2}}.\nonumber
\end{align}
which is the lemma.

\end{proof}

\section{Experimental Details and Further Evaluation}
\label{app:exp}

\begin{table*}[t!]
	\small
	\centering
	\caption{\label{tab:datasets}Summary of learning parameters used for evaluation. For each dataset, batch size (b), and the number of training steps per communication round (r) are presented.} 
		\begin{tabular}{ c c c r r r r }
			\toprule
			\multirow{2}{*}{\textsc{Task}} & \multirow{2}{*}{\textsc{Dataset}} & \multirow{2}{*}{\textsc{Model}} & \multirow{2}{*}{b} & \multirow{2}{*}{r} & \textsc{Training} & \textsc{Total} \\
			& & & & &\textsc{Samples} & \textsc{Parameters} \\ 
			\toprule
            Image Classification & \cifar & GN-LeNet & \num{8} & \num{3} & \num{50000} & \num{89834} \\
            Image Classification & \femnist & CNN & \num{16} & \num{7} & \num{734463} & \num{1690046} \\
			\bottomrule
	\end{tabular}%
	\vspace{1mm}
\end{table*}

\subsection{Experiment setup and learning rate tuning}
\label{sec:exp_setup}

\begin{figure*}[t!]
	\centering
	\inputplot{plots/lrtuning.tex}{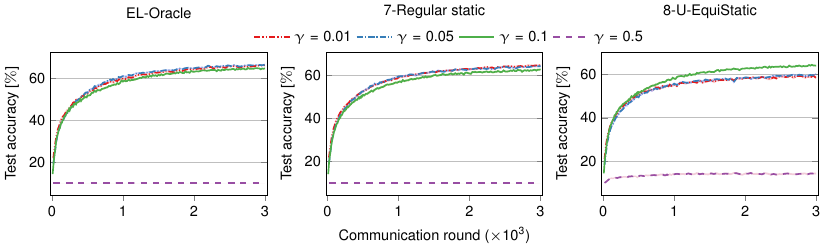}
	\caption{The test accuracy for \elo, \texttt{7-Regular static} topologies, and \texttt{8-U-EquiStatic} topologies with different learning rates $ \gamma $, with the \cifar dataset.}
	\label{fig:lrtuning}
\end{figure*}

\Cref{tab:datasets} provides a summary of the learning parameters, model, and dataset used in the experiments.
The models were evaluated on samples from the test set that were unseen during training.
The step size ($\gamma$) was tuned by running each baseline on a range of values and taking the setup with the best validation accuracy.
The plots in~\cref{fig:lrtuning} show the \cifar convergence plots for \elo, \texttt{7-Regular static}, and \texttt{8-U-EquiStatic} (most prominent topologies) over different step sizes.
The optimal step-sizes are $\gamma = 0.1$ for \texttt{Fully connected} and \texttt{8-U-EquiStatic}, and $\gamma = 0.05$ for the remaining algorithms and topologies over \cifar.
For \femnist, the optimal step size is $\gamma = 0.1$ for all algorithms and topologies.
All experiments in \Cref{sec:eval} were conducted for a total of \num{3000} communication rounds and the experiments for homogeneous data distributions of \cifar (\Cref{fig:eliid}) were conducted for a total of \num{1000} communication rounds.
The experiments over the \femnist dataset were conducted for a total of \num{1440} communication rounds.

\subsection{Computational resources}
\label{app:compute}
We perform experiments on 6 hyperthreading-enabled machines with dual Intel(R) Xeon(R) CPU E5-2630 v3 @ 2.40GHz of 8 cores.
All algorithms were implemented as part of Decentralizepy~\cite{dhasade:2023:dcpy}.
Each DL node was assigned 2 virtual cores.
An experiment took at most 3 hours in wall-clock time for experiments in \Cref{sec:eval}, 2 hours in wall-clock time for \femnist experiments, and 1 hour in wall-clock time for \cifar experiments in \Cref{app:exp}, summing to a total of 1152 virtual CPU hours.
Across all experiments that are presented in this article, including different seeds and learning rate tuning, the total virtual CPU time goes up to approximately \num{54700} hours.

\begin{figure}[tb!]
	\centering
	\inputplot{plots/eliid}{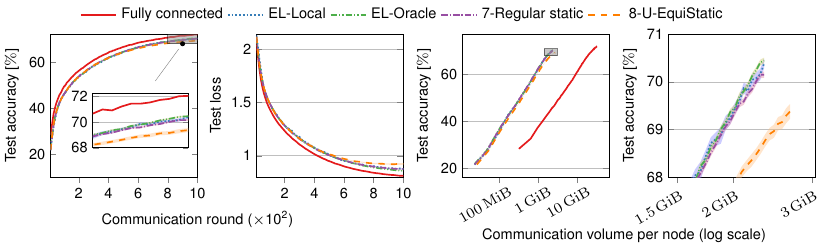}
	\caption{Convergence and communication usage results for homogeneous data partitioning (Dirichlet distribution $\alpha=1.0$) of the \cifar dataset.}
	\label{fig:eliid}
\end{figure}

\subsection{Experiments with homogeneous data distributions}
In this experiment, we focused on evaluating the performance of different communication topologies under \ac{IID} data. The data was generated using a Dirichlet distribution with $\alpha=1.0$.
As a reference baseline, we established a \texttt{Fully connected} topology.
Additionally, we examined \ac{EL} (\elo and \ello) in comparison to \texttt{7-Regular static} and \texttt{8-U-EquiStatic}.

We observed that the results were not particularly exciting (\Cref{fig:eliid}), primarily due to the nature of the data distribution, which leaves limited room for improvement across the different topologies.
In terms of accuracy, the \texttt{Fully connected} topology achieved roughly \SI{2}{\%} better accuracy than the sparser topologies.
\elo displayed a marginal advantage over the other approaches. 
However, the convergence behavior for \ello, \texttt{7-Regular static}, and \elo was nearly overlapping, both in terms of the number of rounds and communication usage.
On the other hand, the \texttt{8-U-EquiStatic} approach lagged behind, exhibiting \SI{1}{\%} lower accuracy.

Although significant improvements were not observed in \ac{IID}  data settings, it is worth noting that \ac{EL} did not hinder convergence. In more realistic \ac{non-IID} data distributions, \ac{EL} demonstrated improved accuracy and reduced the network utilization, as shown in \Cref{fig:elcomparison}, \Cref{sec:eval}.

\subsection{Performance of \ac{EL} and baselines on \femnist}
\label{sec:exp_femnist}

\begin{figure}[tb!]
	\vspace{-0.3cm}
	\centering
	\inputplot{plots/elcomparison-femnist}{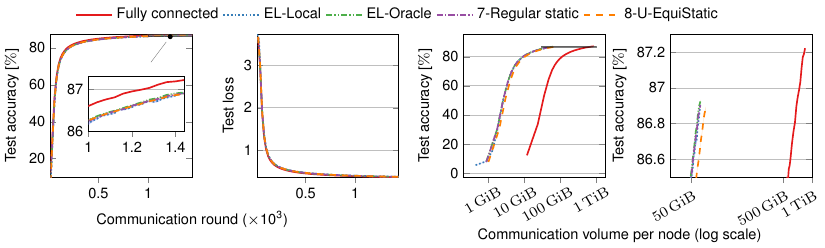}
	\caption{Communication rounds vs. top-1 test accuracy and (left) and communication volume per node vs. test accuracy (right) \newtext{for the \femnist dataset}.}
	\label{fig:elcomparison:femnist}
	\vspace{-0.3cm}
\end{figure}

As an additional dataset, we employed \femnist~\cite{caldas2018leaf} to evaluate \ac{EL} against \dpsgd over other topologies.
\Cref{fig:elcomparison:femnist} shows the convergence plots for the baselines against both \elo and \ello.
We observe the same trend with \texttt{Fully connected} topology performing the best in terms of achieved top-1 accuracy.
\elo outperforms \ello and both achieve higher top-1 accuracy compared to the static topologies \texttt{7-Regular static} and \texttt{8-U-EquiStatic}.
The improvement in accuracy with \ac{EL} (\SI{0.2}{\%} between \elo and \texttt{7-Regular static}) does not look as evident as with the \cifar dataset because the room for improvement (difference between \texttt{7-Regular static} and \texttt{Fully connected}) is quite small.
This can be attributed to the homogeneity of the \femnist dataset and the complexity of the digit recognition task.

\newtext{
\section{Notes on Network Connectivity and \ac{EL} Performance}
\label{sec:network_connectivity}
The implementation of our decentralized scheme is built around the condition that all nodes can communicate with each other.
This is similar to the assumptions of the EquiTopo topologies, a competitor baseline~\cite{song2022communicationefficient}.
However, we argue that the connectivity requirement is a bit more lenient, allowing our \ac{EL} approach to function in a wide range of practical scenarios.
In data center settings, it is common to train on clusters of highly interconnected GPUs, and all-to-all communication should be achievable in these settings.
In edge settings, \eg, a network of mobile devices collaboratively training a model while keeping private datasets, the communication barrier might appear more substantial.
Nonetheless, Internet networks are generally well-connected, which mitigates this concern.
More importantly, from a practical point of view, even if pairwise communications encounter some barriers, the decentralized and randomized nature of \elo and \ello should still allow for effective model learning and convergence.
The occasional lack of communication between specific nodes should not significantly impact the algorithm's performance, as model updates are still propagated through other communicating nodes, as long as the network is not partitioned.

\ac{EL} is most useful in scenarios where every pair of nodes can communicate, but the total communication budget is limited.
Our randomized communication scheme allows for efficient use of the limited resources while ensuring faster model convergence than conventional decentralized learning approaches.}

\end{document}